\documentclass{article}

\PassOptionsToPackage{numbers, compress}{natbib}

\usepackage[preprint]{neurips_2021}

\def\fighome{figures}
\def\bibhome{bib}
\def\texhome{tex}

\usepackage{times}
\usepackage[colorlinks=true,linkcolor=blue,citecolor=green,urlcolor=black]{hyperref}
\usepackage[T1]{fontenc}    
\usepackage{url}            
\usepackage{booktabs}       
\usepackage{nicefrac}       
\usepackage{microtype}      
\usepackage{xcolor}         

\usepackage{floatrow}
\usepackage{caption}
\usepackage{subcaption}

\usepackage{graphicx,float,pgfplots,wrapfig,sidecap,lipsum}
\usepackage{tabularx}
\usepackage{tabulary}
\usepackage{paralist}

\usepackage{array}
\usepackage{makecell}
\usepackage{nomencl}

\usepackage{enumerate}
\usepackage{enumitem}

\usepackage{color}
\usepackage{xcolor}

\usepackage{tablefootnote}
\usepackage{multirow}
\usepackage{diagbox}

\usepackage{xspace}
\usepackage{pifont}
\usepackage{dsfont}
\usepackage{bm}
\usepackage{bbm}
\usepackage{amsfonts}
\usepackage{amsthm}
\usepackage{amssymb}
\usepackage{amsmath}
\usepackage{thmtools,thm-restate}
\usepackage{mathrsfs}
\usepackage{mathtools}
\usepackage{cleveref}

\usepackage{stmaryrd}
\usepackage[utf8]{inputenc}
\usepackage[english]{babel}

\usepackage[utf8]{inputenc}
\usepackage{pgfplots}
\pgfplotsset{compat=newest}
\usepgfplotslibrary{groupplots}
\usepgfplotslibrary{dateplot}

\usepackage{tikz}
\usetikzlibrary{fit}
\usetikzlibrary{calc,shapes}
\usetikzlibrary{decorations.pathmorphing} 
\usetikzlibrary{fit}					
\usetikzlibrary{backgrounds}	

\newtheorem{theorem}{Theorem}[section]

\newtheorem{definition}[theorem]{Definition}
\newtheorem{proposition}[theorem]{Proposition}

\numberwithin{equation}{section}

\newcommand{\mymatrix}[1]{\bm{#1}}
\newcommand{\myvector}[1]{\bm{#1}}

\newcommand{\matrixSup}[2]{\bm{#1}^{(#2)}}

\newcommand{\signal}[1]{\bm{#1}}
\newcommand{\signalSup}[2]{\bm{#1}^{#2}}

\newcommand{\signalX}[2]{\bm{#1}[#2]}
\newcommand{\signalSupX}[3]{\bm{#1}^{#2}[#3]}
\newcommand{\signalSubX}[3]{{#1}_{#2}[#3]}

\newcommand{\Signal}[1]{\bm{#1}}

\newcommand{\SignalX}[2]{\bm{#1}(#2)}

\newcommand{\SignalSupX}[3]{\bm{#1}^{#2}(#3)}

\newcommand{\filter}[1]{\bm{#1}}
\newcommand{\filterSup}[2]{\bm{#1}^{#2}}
\newcommand{\filterSub}[2]{{#1}_{#2}}

\newcommand{\filterX}[2]{\bm{#1}[#2]}
\newcommand{\filterSupX}[3]{\bm{#1}^{#2}[#3]}
\newcommand{\filterSubX}[3]{{#1}_{#2}[#3]}

\newcommand{\FilterX}[2]{\bm{#1}(#2)}
\newcommand{\FilterSupX}[3]{\bm{#1}^{#2}(#3)}

\newcommand{\range}[1]{[#1]}

\def\identity{\bm{I}}
\def\zero{\bm{0}}

\newcommand{\tr}[1]{\mathsf{tr}(#1)}
\newcommand{\Tr}[1]{\mathsf{tr}\left(#1\right)}

\newcommand{\cir}[1]{\mathsf{Cir}(#1)}
\newcommand{\Cir}[1]{\mathsf{Cir}\left(#1\right)}

\newcommand{\norm}[1]{\|#1\|}
\newcommand{\Norm}[1]{\left\|#1\right\|}

\newcommand{\Product}[2]{\left<#1, #2\right>}

\newcommand{\blkdiag}[1]{\mathsf{blkdiag}(#1)}
\newcommand{\Blkdiag}[1]{\mathsf{blkdiag}\left(#1\right)}

\newcommand{\polyphase}[2]{#1|#2}

\newcommand{\adjoint}[1]{{#1}^{\dagger}}

\def\R{\mathbb{R}} 
\def\C{\mathbb{C}} 
\def\Z{\mathbb{Z}} 

\newcommand{\im}{\mathsf{j}} 

\newcommand{\ZTLong}{Z-transform\xspace}
\newcommand{\ZTlong}{z-transform\xspace}

\newcommand{\DTFT}{DTFT\xspace}

\newcommand{\DTFTLong}{Discrete-time Fourier transform\xspace}

\newcommand{\DFT}{DFT\xspace}

\newcommand{\DFTlong}{discrete Fourier transform\xspace}

\newcommand{\RKO}{RKO\xspace}

\newcommand{\OSSN}{OSSN\xspace}

\newcommand{\SVCM}{SVCM\xspace}
\newcommand{\SVCMLONG}{Singular Value Clipping and Masking\xspace}
\newcommand{\SVCMLong}{Singular value clipping and masking\xspace}

\newcommand{\BCOP}{BCOP\xspace}
\newcommand{\BCOPLONG}{Block Convolution Orthogonal Parameterization\xspace}
\newcommand{\BCOPLong}{Block convolution orthogonal parameterization\xspace}

\newcommand{\CayleyConv}{CayleyConv\xspace}
\newcommand{\CayleyConvLONG}{Cayley Convolution\xspace}
\newcommand{\CayleyConvLong}{Calyey convolution\xspace}

\newcommand{\SCFac}{SC-Fac\xspace}
\newcommand{\SCFacLONG}{Separable Complete Factorization\xspace}

\newcommand{\SCFaclong}{separable complete factorization\xspace}

\newcommand{\ResNet}{ResNet\xspace}
\newcommand{\WRN}{WideResNet\xspace}
\newcommand{\ShuffleNet}{ShuffleNet\xspace}

\newcommand{\CNN}{CNN\xspace}
\newcommand{\CNNLONG}{Convolutional Neural Network\xspace}

\newcommand{\RNN}{RNN\xspace}

\newcommand{\tcr}[1]{\textcolor[rgb]{1,0,0}{#1}}
\newcommand{\tcy}[1]{\textcolor[rgb]{0.14, 0.53, 0.25}{#1}} 
\newcommand{\tcb}[1]{\textcolor[rgb]{0,0,1}{#1}}

\def\mytitle{Scaling-up Diverse Orthogonal Convolutional Networks with a Paraunitary Framework}
\title{\mytitle}

\author{%
	\begin{tabular}{c} Jiahao Su$^{1}$ \\ {\tt \href{mailto:jiahaosu@umd.edu}{jiahaosu@umd.edu}} \end{tabular} \begin{tabular}{c} Wonmin Byeon$^{2}$ \\ {\tt \href{mailto:wbyeon@nvidia.com}{wbyeon@nvidia.com}} \end{tabular} \begin{tabular}{c} Furong Huang$^{1}$ \\ {\tt \href{mailto:furongh@cs.umd.edu}{furongh@cs.umd.edu}} \end{tabular} \\
  	$^{1}$University of Maryland, College Park
	\hskip0.5em $^{2}$NVIDIA Research, NVIDIA Corporation \\
}

\begin{document}

\maketitle

\begin{abstract}
Enforcing orthogonality in neural networks is an antidote for gradient vanishing/exploding problems, sensitivity by adversarial perturbation, and bounding generalization errors.
However, many previous approaches are heuristic, and the orthogonality of convolutional layers is not systematically studied: some of these designs are not exactly orthogonal, while others only consider standard convolutional layers and propose specific classes of their realizations.
To address this problem, we propose a theoretical framework for orthogonal convolutional layers, which establishes the equivalence between various orthogonal convolutional layers in the spatial domain and the paraunitary systems in the spectral domain.
Since there exists a complete spectral factorization of paraunitary systems, any orthogonal convolution layer can be parameterized as convolutions of spatial filters.
Our framework endows high expressive power to various convolutional layers while maintaining their exact orthogonality.
Furthermore, our layers are memory and computationally efficient for deep networks compared to previous designs.
Our versatile framework, for the first time, enables the study of architecture designs for deep orthogonal networks, such as choices of skip connection, initialization, stride, and dilation. Consequently, we scale up orthogonal networks to deep architectures, including ResNet, WideResNet, and ShuffleNet, substantially increasing the performance over the traditional shallow orthogonal networks.
\end{abstract}

\section{Introduction}
\label{sec:introduction}

{\CNNLONG}s ({\CNN}s), whose deployment has witnessed extensive empirical success, still exhibit a range of limitations that are not thoroughly studied.
For example, firstly, deep convolutional networks are in general difficult to learn, and their high performance heavily relies on not yet fully-understood techniques, such as skip-connections~\citep{he2016deep}, batch normalization~\citep{ioffe2015batch}, delicate initialization~\citep{glorot2010understanding}.
Secondly, they are notoriously sensitive to imperceptible perturbations,
including adversarial attacks~\citep{goodfellow2014explaining} or natural geometric transformations~\citep{azulay2019deep}.
Finally, a precise characterization of their generalization property is still under active investigation~\citep{neyshabur2017pac,jia2019orthogonal}.

Orthogonal neural networks alleviate all the problems mentioned above. As shown in recent works,
by enforcing orthogonality in neural networks, i.e., each layer's output norm $\norm{\signal{y}}$ is always equal to the input norm $\norm{\signal{x}}$ for any input $\signal{x}$, \emph{we obtain (1) easier optimization}~\citep{zhang2018stabilizing,qi2020deep} since each orthogonal layer, by definition, preserves the gradient norm during backpropagation, and the whole network is free from the vanishing and exploding gradient problems;
\emph{ (2) robustness against adversarial perturbation}~\citep{anil2019sorting,li2019preventing,trockman2021orthogonalizing} since each orthogonal layer has a Lipschitz constant strictly less than one --- the network can not amplify any perturbation to the input to flip the output prediction;
\emph{ and (3) better generalizability} as proved in \citep{jia2019orthogonal} that a deep network's generalization error is positively related to the standard deviation of each linear layer's singular values, and thus, an orthogonal network has the smallest generalization error.
Note that an orthogonal layer is guaranteed to have ``flat'' singular values, i.e., the singular value for each spectral component is equal to 1.

Our goal is to enforce exact orthogonality in state-of-the-art deep convolutional networks without expensive computations. 
We identify three main challenges in doing so.
\textbf{Challenge I: Achieving exact orthogonal throughout the entire training process.}
Exact orthogonality is crucial for applications that require strict characterization of Lipschitz constants, such as robustness against adversarial perturbations. 
Prior works such as {\em soft regularization}~\citep{jia2017improving,wang2019orthogonal} and {\em reshaped kernel orthogonality}~\citep{jia2017improving,cisse2017parseval}, while enjoying algorithmic simplicity, are not sufficient to address the requirement of strict orthogonality.
\textbf{Challenge II: Avoiding expensive computations.}
An efficient training algorithm is crucial for scalability to large networks or datasets.
Existing works based on projected gradient descent~\citep{sedghi2018singular,li2019preventing}, however, require an expensive projection step after each update, which may also introduce additional challenges such as difficulty in convergence.
For example, the projection step proposed in \citet{sedghi2018singular} computes the SVD and clip the spectrum to obtain ``flattened'' spectrum and thus orthogonality, but SVD of convolution layers are expensive and cost $O(\mathrm{size}(\mathrm{feature}) \cdot \mathrm{channel s}^{3})$.
\textbf{Challenge III: Scaling-up to state-of-the-art deep convolutional networks.} 
There are many variants to the traditional convolutional layer, including dilated, strided, group convolutions, which are essential for state-of-the-art deep convolutional networks.
However, none of the existing methods have proposed mechanisms to orthogonalize these variants.
The lack of techniques, as a result, limits the broad applications in orthogonal convolutional layers to state-of-the-art deeper convolutional networks.

In this work, we resolve these \textbf{challenges I, II and III} by proposing a {\em complete parameterization} for orthogonal convolutional layers --- any realization of the parameters would lead to an orthogonal convolution, and any orthogonal convolution can be realized by some parameters.
Specifically, using the convolution theorem~\citep{oppenheim2010discrete} (which states that convolution in the spatial domain is equivalent to multiplication in the frequency domain),
we reduce the problem of designing orthogonal convolutions to constructing unitary matrices for all frequency components, i.e., a paraunitary system~\citep{vaidyanathan1993multirate}.
To construct the paraunitary system, we parameterize it in a complete factorization form.
Therefore, we obtain a parameterization for the class of {\em all} orthogonal convolutions, guaranteeing the expressive power.
Our parameterization is easier for optimization compared to projected gradient descent and other parameterization based on matrix inversion.
Our versatile framework, for the first time, enables the study of architecture designs for deep orthogonal networks, such as choices of skip connection, initialization, stride, and dilation.
Consequently, we scale up orthogonal networks to deep architectures, including ResNet, WideResNet, and ShuffleNet, substantially increasing the performance over the traditional shallow orthogonal networks.

\vspace{0.2em}
\textbf{Summary of Contributions:}
\begin{enumerate}[leftmargin=*, itemsep=0pt, topsep=0pt]
    \item We establish the equivalence between orthogonal convolutions in the spatial domain and paraunitary systems in the spectral domain, simplifying designing orthogonal convolutions. Consequently, we can interpret the existing approaches as implicit designs of paraunitary systems.
    \item Based on a complete factorization of paraunitary systems, we propose the first exact and complete design for orthogonal convolutions, ensuring exact orthogonality and high expressive power.
    \item We prove that orthogonality for various convolutional layers (strided, dilated, group) are also completely characterized by paraunitary systems. Consequently, our design easily extends to these variants, ensuring both completeness and exactness of the orthogonal convolutions.
    \item We systematically study the design considerations for orthogonal networks (choices of skip connection, initialization, depth, width, and kernel size), and show that orthogonal networks can scale to deep architectures including \ResNet, \WRN, and \ShuffleNet.
\end{enumerate}
\section{Orthogonal Convolutions via Paraunitary Systems}
\label{sec:paraunitary}

\vspace{-0.2em}
\subsection{Achieving Orthogonal Convolutions by Paraunitary Systems}
\label{sub:paraunitary-equivalence}

Designing an orthogonal convolution layer $\{\bm{h}_{t,s}:\signal{y}_t = \signal{h}_{t,s} \ast \signal{x}_s\}_{t=1,s=1}^{T,S}$ ($t$ and $s$ denotes an ouput and input channel) in the spatial domain is challenging; it is equivalent to the problem of making a block-circulant matrix  $[\Cir{\filterSub{\bm{h}}{t,s}}]_{t=1,s=1}^{T,S}$ orthogonal, since the output $\bm{y}_t$ is equivalent to \emph{the circulant structure of $\bm{h}_{t,s}$} multiplied with the input $\bm{x}_s$, i.e.
$\bm{y}_t = \cir{\filterSub{\bm{h}}{t,s}} \bm{x}_s$ for all $s$, $t$, where
\begingroup
\small
\begin{equation}
\Cir{\filterSub{\bm{h}}{t,s}}   =
\begin{bmatrix}
\tcy{\filterSubX{h}{t,s}{1}} & \tcb{\filterSubX{h}{t,s}{N}} & \cdots & \tcr{\filterSubX{h}{t,s}{2}}
\\
\tcr{\filterSubX{h}{t,s}{2}} & \tcy{\filterSubX{h}{t,s}{1}} & \tcb{\filterSubX{h}{t,s}{N}} & \cdots 
\\
\vdots & \ddots & \ddots & \vdots
\\
\tcb{\filterSubX{h}{t,s}{N}} & \cdots & \tcr{\filterSubX{h}{t,s}{2}} & \tcy{\filterSubX{h}{t,s}{1}}
\\
\end{bmatrix} \in \mathbb{R}^{N \times N}\label{eq:conv1d-circulant}.
\end{equation}
\endgroup

To avoid enforcing orthogonality directly in the block-circulant structures, we propose a novel design of orthogonal convolutions from a spectral domain perspective, motivated by the \emph{convolution theorem} (\Cref{thm:conv-theorem}).
For simplicity, we group the entries at the same locations into a vector/matrix, e.g., we denote $\{\signalSubX{x}{s}{n}\}_{s = 1}^{S}$ as $\signalX{x}{n} \in \R^{S}$ and $\{\filterSubX{h}{t,s}{n}\}_{t=1,s=1}^{T,S}$ as $\filterX{h}{n} \in \R^{T \times S}$.

\begin{restatable}[Convolution theorem~\citep{oppenheim2010discrete}]{theorem}{ThmConv}
\label{thm:conv-theorem}
For a standard convolution layer $\filter{h}$: $\signalX{y}{i} = \sum_{n} \filterX{h}{n} \signalX{x}{i - n}$, the convolution in the spatial domain is equivalent to a matrix-vector multiplication in the spectral domain:  
$\SignalX{Y}{z} = \FilterX{H}{z} \SignalX{X}{z}, \forall z\in \mathbb{C}$,
where $\SignalX{X}{z} = \sum_{n=0}^{N-1} \signalX{x}{n} z^{-n}$, $\SignalX{Y}{z} = \sum_{n=0}^{N-1} \signalX{y}{n} z^{-n}$, $\FilterX{H}{z} = \sum_{n = -\underline{L}}^{\overline{L}} \SignalX{h}{z}$,
denote the input, output and kernel in the spectral domain, also known as $z$-transforms, where $N$ is the length of $\signal{x}, \signal{y}$ and $[-\underline{L}, \overline{L}]$ is the span of the filter. 
\end{restatable}

The convolution theorem states that a standard convolution layer is a matrix-vector multiplication in the spectral domain (without expanding the kernel into a block-circulant structure).
As long as the \emph{transfer matrix} $\FilterX{H}{z}$
is unitary at $z = e^{\im \omega}$ for all frequencies $\forall \omega \in \mathbb{R}$ ($\im$ is the imaginary unit), the convolution $\filter{h}$ is orthogonal.

Therefore, as a major novelty of our paper, 
we design orthogonal convolutions through designing unitary transfer matrix $\FilterX{H}{e^{\im\omega}}$ at all frequencies $\omega \in \R$, which is known as a \emph{paraunitary systems}~\citep{vaidyanathan1993multirate,strang1996wavelets}.
We prove in \Cref{thm:paraunitary-conv1d} that a convolutional layer is orthogonal in the spatial domain if and only if it is paraunitary in the spectral domain.

\vspace{0.2em}
\textbf{Benefits through paraunitary systems.}
\textbf{(1)}
The spectral representation simplifies the designs of orthogonal convolutions, which avoids dealing with the block-circulant structures.
\textbf{(2)} 
Due to the equivalence between orthogonal convolutions and paraunitary systems, i.e., one is sufficient and necessary for the other, it is {\em impossible} to find an orthogonal convolutional layer whose transfer matrix is {\em not} paraunitary and vice versa.
\textbf{(3)} 
There exists a complete factorization of paraunitary systems: any paraunitary $\FilterX{H}{z}$ can be realized through a factorization in the spectral domain, as will be shown in \Cref{eq:paraunitary-factorization-maintext}.
\textbf{(4)} 
Since multiplicative factorization in the spectral domain corresponds to convolutions in the spatial domain, we can realize \emph{any} orthogonal convolution as a convolution of multiple sub-layers, each of which is parameterized by orthogonal matrices.
\textbf{(5)} 
Finally, there are mature reparameterization methods that realize orthogonal matrices using unconstrained parameters. Therefore, we realize orthogonal convolutions through standard neural network optimizers on a model parameterized via our design.

\vspace{0.2em}
{\bf Interpretation of existing methods.}
Since paraunitary system is a necessary and sufficient condition for orthogonal convolution, existing approaches, including {\em \SVCMLONG (\SVCM)}~\citep{sedghi2018singular}, {\em \BCOPLong (\BCOP)}~\citep{li2019preventing}, {\em \CayleyConvLONG (\CayleyConv)}~\citep{trockman2021orthogonalizing}, can be interpreted from a paraunitary perspective.
Specifically, \SVCM clips the singular values of $\FilterX{H}{e^{\im \omega}}$ to $1$ for each frequency $\omega$, \CayleyConv generalizes the Cayley transform from unitary matrices to paraunitary systems, and \BCOP is a spectral factorization of the paraunitary systems.
We discuss these interpretations in more detail in \Cref{app-sub2:framework-construction-interpretation}.

\subsection{Realizing Paraunitary Systems via Re-parameterization}
\label{sub:paraunitary-factorization}

After reducing the problem of orthogonal convolutions to paraunitary systems, we are left with the question of how to realize paraunitary systems. 
We use a complete factorization form of paraunitary systems to realize any paraunitary systems.

According to \Cref{thm:paraunitary-factorization}, we see that any paraunitary system $\FilterX{H}{z}$ can be written as the form in \Cref{eq:paraunitary-factorization-maintext} and any $\FilterX{H}{z}$ in this form is a paraunitary system.
\begin{equation}
\label{eq:paraunitary-factorization-maintext}
\FilterX{H}{z} = \FilterX{V}{z; \matrixSup{U}{-\underline{L}}} \cdots \FilterX{V}{z; \matrixSup{U}{-1}} \mathbf{Q} \FilterX{V}{z^{-1}; \matrixSup{U}{1}} \cdots \FilterX{V}{z^{-1}; \matrixSup{U}{\overline{L}}},
\end{equation}
 where $\mymatrix{Q}$ is an orthogonal matrix, each $\{\matrixSup{U}{\ell}\}_{\ell=-\underline{L}}^{\overline{L}}$ is a column-orthogonal matrix, and 
 \begin{equation}\label{eq:sublayer-spectral}
      \FilterX{V}{z; {\mymatrix{U}^{(\ell)}}} = (\identity - \mymatrix{U}^{(\ell)} {\mymatrix{U}^{(\ell)}}^{\top}) + \mymatrix{U}^{(\ell)} {\mymatrix{U}^{(\ell)}}^{\top} z, \quad \forall \ell=-\underline{L},\ldots,\overline{L}.
 \end{equation}
This complete factorization form is significantly useful for our design of  orthogonal convolutions.
As multiplications in the spectral domain are equivalent to convolutions in the spatial domain (\Cref{thm:conv-theorem}), any paraunitary system realized by \Cref{eq:paraunitary-factorization-maintext} can be constructed as a sequence of convolutions parameterized by $\FilterX{V}{z; {\mymatrix{U}^{(\ell)}}}$'s spatial counterparts as well as the orthogonal $Q$.

\vspace{0.2em}
{\bf Model design in the spatial domain.}
According to the argument above, any convolutional layer can be parameterized by several (column-)orthogonal matrices.
Precisely, we obtain a \emph{complete design of orthogonal 1D-convolutions}.
Using \textbf{learnable column-orthogonal matrices} $(\mymatrix{Q},\{\matrixSup{U}{\ell}\}_{\ell=-\underline{L}}^{\overline{L}})$, we parameterize any convolution layer with filter size $\overline{L} + \underline{L}$, as convolutions of $\overline{L}$ number of filters 
\begin{equation}
\left\{\left[\identity - \matrixSup{U}{\ell} {\matrixSup{U}{\ell}}^{\top}, \  \matrixSup{U}{\ell} {\matrixSup{U}{\ell}}^{\top}\right]\right\}_{\ell=-\overline{L}}^{1},
\end{equation}
followed by a $1 \times 1$ convolution with $\mymatrix{Q}$, and then convolutions of $\underline{L}$ number of filters
 \begin{equation}
     \left\{\left[\matrixSup{U}{\ell} {\matrixSup{U}{\ell}}^{\top}, \identity - \matrixSup{U}{\ell} {\matrixSup{U}{\ell}}^{\top}\right]\right\}_{\ell=1}^{\underline{L}}.
 \end{equation}
\Cref{fig:paraunitary} provides a visualization of our proposed design of orthogonal convolution layers; each block denotes a convolution and the form of the filter is displayed in each of the block.

\begin{figure}[htbp]
    \centering
    \includegraphics[trim={0.5cm 0.7cm 0.5cm 0.7cm},clip,width=\textwidth]{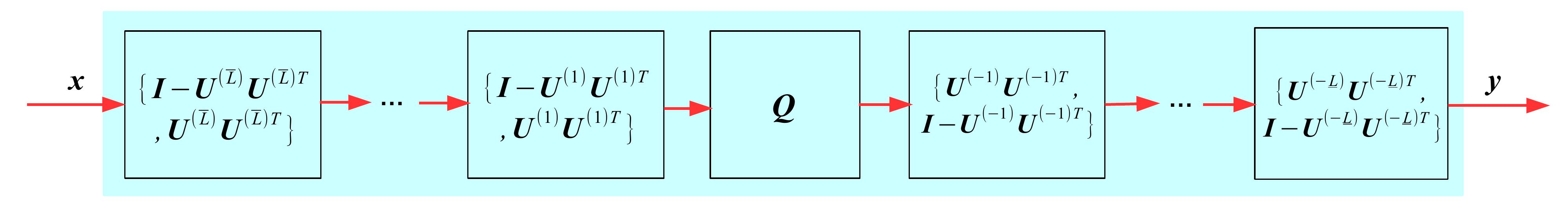}
    \caption{
    Complete design of orthogonal convolution layer as a cascade of convolutions, whose filter parameterization is depicted in each block.  $\bm{Q}$ and $\{\bm{U}^{(\ell)}\}$  are orthogonal matrices.}
    \label{fig:paraunitary}
\end{figure}

Using this complete design of 1D-convolution, we obtain a complete design for {\em separable} orthogonal 2D-convolutions.
We parameterize any separable 2D-convolution of filter size $(\overline{L} + \underline{L} + 1) \times (\overline{R} + \underline{R} + 1)$, as a convolution of two orthogonal 1D-convolutions
parameterized with learnable column-orthogonal matrices $(\mymatrix{Q}_{1}, \{\matrixSup{U}{\ell}_{1}\}_{\ell=-\underline{L}}^{\overline{L}})$ and $(\mymatrix{Q}_{2}, \{\matrixSup{U}{r}_2\}_{r=-\underline{R}}^{\overline{R}})$ respectively.

With a complete factorization of paraunitary systems, we reduce the problem of designing orthogonal convolutions to the one for orthogonal matrices.

\vspace{0.2em}
\textbf{Parameterization for orthogonal matrices.} 
In \Cref{app-sub:framework-parameterization}, we perform a comparative study on different parameterizations of orthogonal matrices, including the {\em Björck orthogonalization}~\citep{anil2019sorting,li2019preventing}, the {\em Cayley transform}~\citep{helfrich2018orthogonal,maduranga2019complex}, and the {\em Lie exponential map}~\citep{lezcano2019cheap,lezcano2019trivializations}.
We adopt the Lie exponential map since it is the only one that provides an exact and complete characterization of all orthogonal matrices.
The Lie exponential map is a {\em surjective} mapping from a skew-symmetry matrix $\mymatrix{A}$ to an orthogonal matrix $\mymatrix{U}$ with
$\mymatrix{U} = \exp(\mymatrix{A}) = \identity + \mymatrix{A} + \frac{1}{2} \mymatrix{A}^2 + \cdots$, where the infinite sum can be computed exactly up to machine-precision~\citep{higham2009scaling}.

\begin{figure}
    \centering
    \includegraphics[trim={0.5cm 1cm 0.5cm 1cm},clip,width=\textwidth]{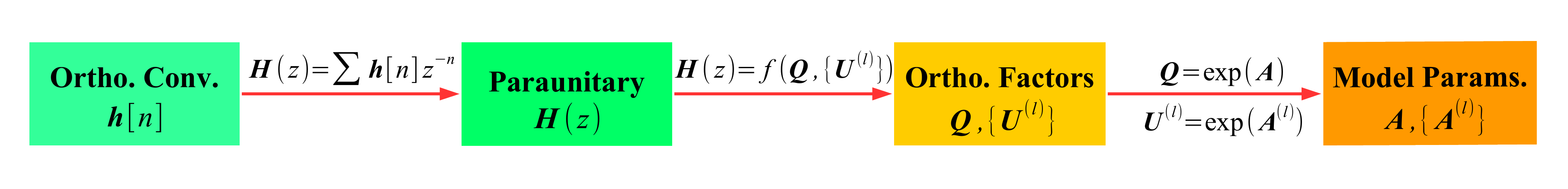}
    \caption{
    {\bf \SCFac: A pipeline for designing orthogonal convolutional layer.}
    {\bf(1)} An orthogonal convolution $\signalX{h}{n}$ is equivalent a paraunitary system $\SignalX{H}{z}$ in the  spectral domain (\Cref{thm:conv-theorem}).
    {\bf(2)} The paraunitary system $\SignalX{H}{z}$ is multiplications of factors characterized by orthogonal matrices $\mymatrix{Q}$ and $\{\matrixSup{U}{\ell}\}$ (\Cref{eq:paraunitary-factorization-maintext}, \Cref{thm:paraunitary-conv1d}).
    {\bf(3)} These orthogonal matrices are parameterized by skew-symmetric matrices using Lie exponential map.
    }
    \label{fig:pipeline}
\end{figure}

Now we have an end-to-end pipeline as shown in \Cref{fig:pipeline} of our proposed design of orthogonal convolutional layers. Since the major component of our design uses a complete factorization for separable paraunitary systems, we call our design {\em \SCFacLONG (\SCFac)}.
\section{Unifying Orthogonal Convolution Variants as Paraunitary Systems}
\label{sec:framework}

Various convolutional layers (strided, dilated, and group convolution) are widely used in neural networks.
However, it is not apparent how to enforce their orthogonality, as the convolution theorem (\Cref{thm:conv-theorem}) only holds for \emph{standard} convolutions. 
Previous approaches only deal with standard convolutions~\citep{sedghi2018singular,li2019preventing,trockman2021orthogonalizing}, thus orthogonality for state-of-the-art architectures are never studied before.

We address this limitation by modifying convolution theorem for each variant of convolution layer, which allows us to design these variants using paraunitary systems.

\begin{theorem}[Convolution and paraunitary theorems for various convolutions]
\label{thm:conv-theorem-variants}
Strided, dilated, and group convolutions can be unified in the spectral domain as $\FilterX{\underline{Y}}{z} = \FilterX{\underline{H}}{z} \FilterX{\underline{X}}{z}$,
where $\FilterX{\underline{Y}}{z}$, $\FilterX{\underline{H}}{z}$, $\FilterX{\underline{X}}{z}$ are modified {\ZTLong}s of $\signal{y}$, $\filter{h}$, $\signal{x}$, where the modified {\ZTLong}s are instantiated for strided convolutions in \Cref{prop:strided-conv1d}, dilated convolution in \Cref{prop:dilated-conv1d}, and group convolution in \Cref{prop:group-conv1d}.
Furthermore, a convolution is orthogonal if and only if $\FilterX{\underline{H}}{z}$ is paraunitary. 
\end{theorem}

\begin{wrapfigure}[15]{r}{0.65\textwidth}
    \vspace{-1em}
    \begin{subfigure}[b]{0.48\textwidth}
        \fbox{\includegraphics[trim={1.4cm 1.4cm 1.4cm 1.4cm },clip,width=0.9\textwidth]{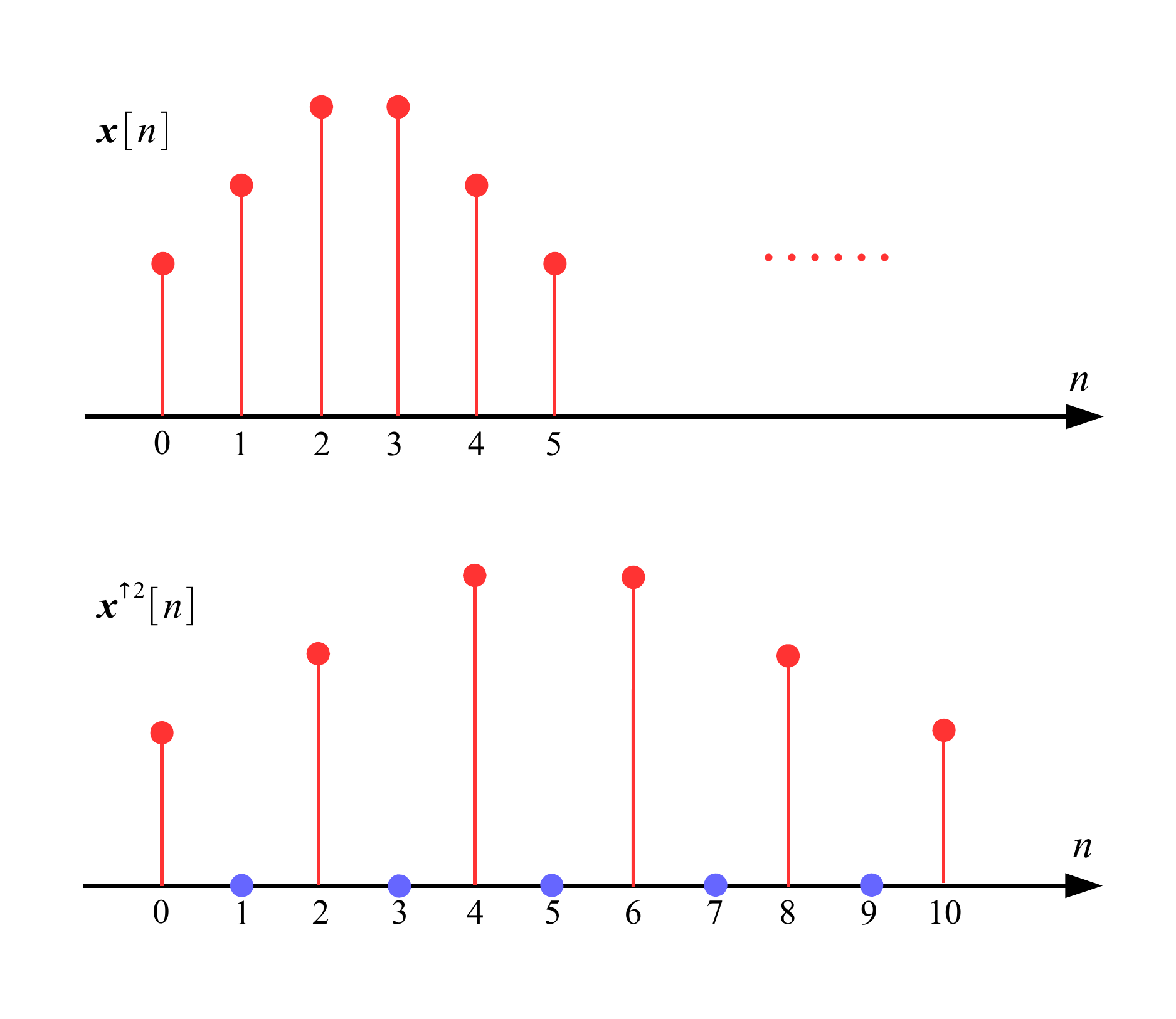}}
        \caption{Up-sample}
        \label{subfig:up-sampling}
    \end{subfigure}
    \hfill
    \begin{subfigure}[b]{0.48\textwidth}
        \fbox{\includegraphics[trim={1.4cm 1.4cm 1.4cm 1.4cm },clip,width=0.9\textwidth]{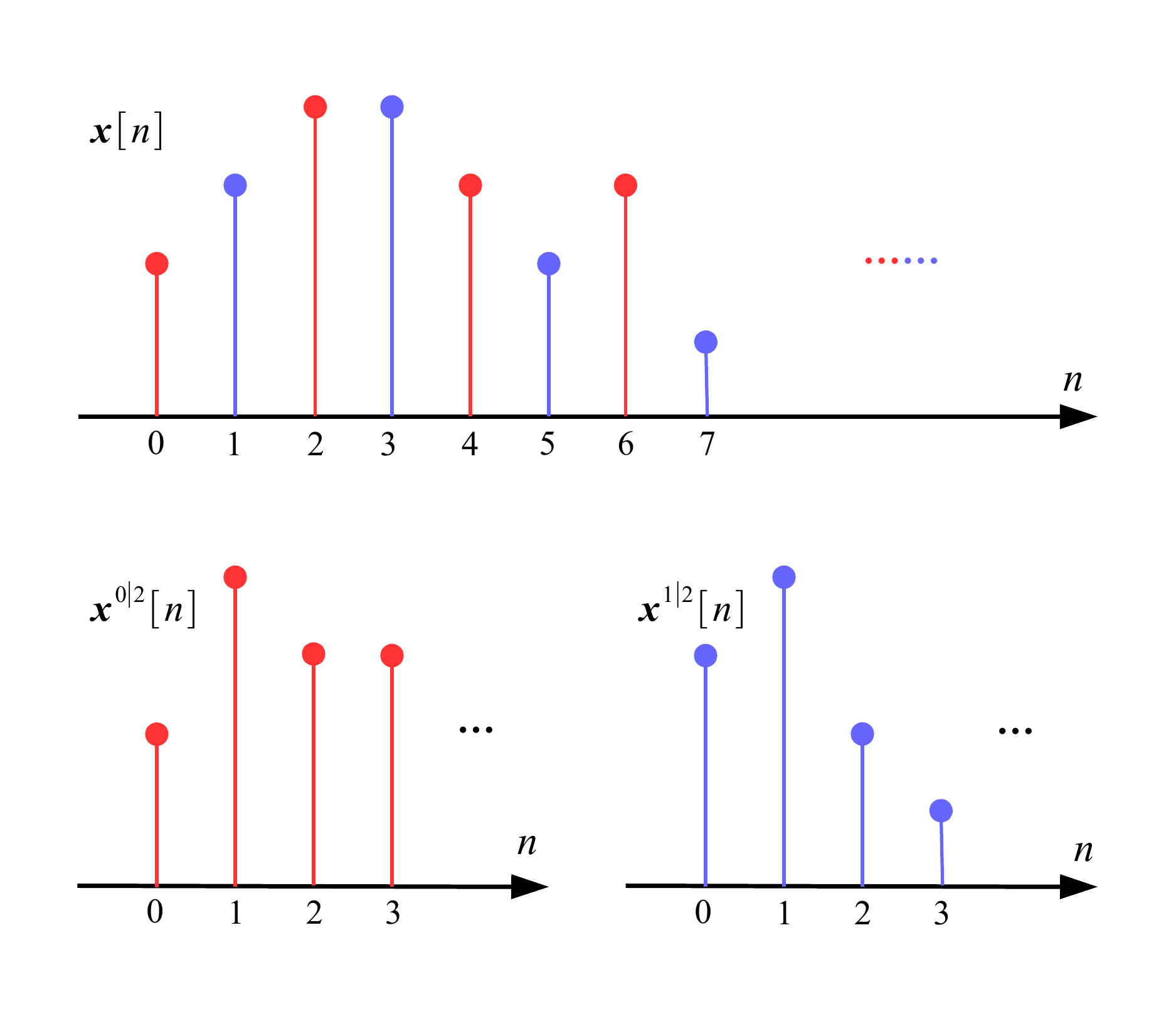}}
        \caption{Down-sample}
        \label{subfig:polyphase-components}
    \end{subfigure}
\caption{\small{{\bf Up and down sampling.}
In {\bf (a)}, the sequence $\signalX{x}{n}$ is up-sampled into $\signalSupX{x}{\uparrow 2}{n}$.
In {\bf (b)}, $\signalX{x}{n}$ is down-sampled into  $\signalSupX{x}{\polyphase{0}{2}}{n}$ with even entries (\textcolor{red}{red}) 
and $\signalSupX{x}{\polyphase{1}{2}}{n}$ with odd entries(\textcolor{blue}{blue}).
}}
\label{fig:multi-resolution}
\end{wrapfigure}

In~\Cref{tab:comparison-conv-layers}, we formulate strided, dilated, and group convolutions in the spatial domain, interpreting them as up-sampled or down-sampled variants of a standard convolution. Now, we introduce the concept of up-sampling and down-sampling precisely below.

Given a sequence $\signal{x}$,
we introduce its \emph{up-sampled sequence} $\signalSup{x}{\uparrow R}$ with sampling rate $R$ as $\signalSupX{x}{\uparrow R}{n} \triangleq 
\signalX{x}{n / R}$ for $n \equiv 0\!\pmod{R}$.
On the other hand, its \emph{($r$,$R$)-polyphase component} $\signalSup{x}{\polyphase{r}{R}}$ indicates the $r$-th down-sampled sequence with sampling rate $R$, defined as $\signalSupX{x}{\polyphase{r}{R}}{n} \triangleq \signalX{x}{n R + r}$.
We illustrated an example of $\signalSup{x}{\uparrow R}$ and $\signalSup{x}{\polyphase{r}{R}}$ in \Cref{fig:multi-resolution} when sampling rate $R = 2$. The {\ZTLong}s of $\signalSup{x}{\uparrow R}$, $\signalSup{x}{\polyphase{r}{R}}$ are denoted as  $\SignalSupX{X}{\uparrow R}{z}$,  $\SignalSupX{X}{\polyphase{r}{R}}{z}$ respectively. Their relations to $\FilterX{X}{z}$ are studied in \Cref{app-sub:multi-resolution}. 

\begin{table}[!htbp]
    \caption{\small{\bf Variants of convolutions.}
    We present the modified $Z$-transforms, ${\color{purple} \SignalX{\underline{Y}}{z}}$, ${\color{red} \FilterX{\underline{H}}{z}}$, and ${\color{blue} \SignalX{\underline{X}}{z}}$ for each convolution such that  ${\color{purple} \SignalX{\underline{Y}}{z}} = {\color{red} \FilterX{\underline{H}}{z}} {\color{blue} \SignalX{\underline{X}}{z}}$ holds.
    In the table, $\SignalSupX{X}{\range{R}}{z} \triangleq [\SignalSupX{X}{\polyphase{0}{R}}{z}^\top,\ldots,\SignalSupX{X}{\polyphase{R-1}{R}}{z}^\top ]^\top$ and 
$\SignalSupX{\widetilde{X}}{\range{R}}{z} = [\SignalSupX{X}{\polyphase{-0}{R}}{z},\ldots,\SignalSupX{X}{\polyphase{-(R-1)}{R}}{z}]$.
    For group convolution, $\filterSup{h}{g}$ is the filter for the $g^{\mathrm{th}}$ group with $\FilterSupX{H}{g}{z}$ being its \ZTLong, and $\Blkdiag{\cdot}$ stacks multiple matrices into a block-diagonal matrix.
    }
    \label{tab:comparison-conv-layers}
    \vspace{-0.5em}
    \centering
    \renewcommand{\arraystretch}{1.3}
    \renewcommand{\tabcolsep}{5pt}
     \resizebox{\textwidth}{!}{
    \begin{tabular}{c | c | c c c}
    \toprule
    \multirow{2}{*}{Convolution Type} & \multirow{2}{*}{Spatial Representation} & \multicolumn{3}{c}{Spectral Representation}\\
    & & ${\color{purple} \SignalX{\underline{Y}}{z}}$ & ${\color{red} \FilterX{\underline{H}}{z}}$ & ${\color{blue} \SignalX{\underline{X}}{z}}$\\ 
    \midrule
     Standard  
     & ${\color{purple} \signalX{y}{i}} = \sum_{n \in \Z} {\color{red} \filterX{h}{n}} {\color{blue} \signalX{x}{i - n}}$
     & ${\color{purple} \SignalX{Y}{z}}$ & ${\color{red} \FilterX{H}{z}}$ & ${\color{blue} \SignalX{X}{z}}$ 
     \\
     $R$-Dilated 
     & ${\color{purple} \signalX{y}{i}} = \sum_{n \in \Z} {\color{red} \filterSupX{h}{\uparrow R}{n}} {\color{blue} \signalX{x}{i - n}}$
     & ${\color{purple} \SignalX{Y}{z}}$ & ${\color{red} \SignalX{H}{z^R}}$ & $ {\color{blue} \SignalX{X}{z}}$
     \\
     $\downarrow\!R$-Strided 
     & ${\color{purple} \signalX{y}{i}} = \sum_{n \in \Z} {\color{red} \filterX{h}{n}} {\color{blue} \signalX{x}{R i - n}}$ 
     & ${\color{purple} \FilterX{Y}{z}} $ & $ {\color{red} \FilterSupX{\widetilde{H}}{\range{R}}{z}} $ & ${\color{blue} \SignalSupX{X}{\range{R}}{z}}$
     \\
     $\uparrow\!R$-Strided 
     & ${\color{purple} \signalX{y}{i}} = \sum_{n \in \Z} {\color{red} \filterX{h}{n}} {\color{blue} \signalSupX{x}{\uparrow R}{i - n}}$ 
     & ${\color{purple} \SignalSupX{Y}{\range{R}}{z}} $ & $ {\color{red} \FilterSupX{H}{\range{R}}{z}} $ & $ {\color{blue} \SignalX{X}{z}}$
     \\
     $G$-Group 
     & ${\color{purple} \signalX{y}{i}} = \sum_{n \in \Z} {\color{red} \Blkdiag{\{\filterSupX{h}{g}{n}\}}} {\color{blue} \signalX{x}{i - n}}$
     & ${\color{purple} \SignalX{Y}{z}} $ & $ {\color{red} \Blkdiag{\{\FilterSupX{H}{g}{z}\}}} $ & $ {\color{blue} \SignalX{X}{z}}$
     \\ 
    \bottomrule
    \end{tabular}
    }
\end{table}

Now we are ready to interpret convolution variants.
\emph{(1) Strided convolution} is used to adjust the feature resolution: A strided convolution (
$\downarrow\!R$-strided) decreases the resolution by down-sampling after a standard convolution,
while a transposed strided convolutional layer (
$\uparrow\!R$-strided) increases the resolution by up-sampling before a standard convolution.
\emph{(2) Dilated convolution}
 increases the receptive field of a convolution without extra parameters and computation: an $R$-dilated convolution up-samples its filters before convolution with the input.
\emph{(3) Group convolution}
is proposed to reduce the parameters and computations, thus widely used by efficient architectures: a $G$-group convolution divides the input/output channels into $G$ groups and restricts the connections within each group.
In \Cref{app-sub:framework-unification}, we prove that a (variant of) convolution is orthogonal if and only if the modified Z-transform $\FilterX{\underline{H}}{z}$ is paraunitary.
\section{Learning Deep Orthogonal Networks with Lipschitz Bounds}
\label{sec:deep-networks}

In this section, we switch our focus from layer design to network design. In particular, we aim to study how to scale-up deep orthogonal networks with Lipschitz bounds.

{\bf Lipschitz networks}~\citep{anil2019sorting,li2019preventing,trockman2021orthogonalizing}, whose Lipschitz upper bounds are imposed by their architectures, are proposed as competitive candidates to guarantee robustness in deep learning. Orthogonal layers are essential components in these networks, which are $1$- Lipschitz while preserving gradient norm in backpropagation.
Given a Lipschitz constant $L$, a network $f$ can compute a certified radius for each input from its output margin. Formally, denote the output margin of an input $\signal{x}$ with label $c$ as
\begin{equation}
\label{eq:margin}
\mathcal{M}_f(x) \triangleq \max(0, f(\signal{x})_c - \max_{i \neq c} f(\signal{x})_i),
\end{equation}
i.e., the difference between the correct logit and the second largest logit. Then the output is robust to perturbation such that $f(\signal{x} + \signal{\epsilon}) = f(\signal{x}) = c, \forall \epsilon: \norm{\signal{\epsilon}} < \mathcal{M}_f(\signal{x}) / \sqrt{2}L$.

Despite the benefit, existing architectures for Lipschitz networks remain shallow, and a Lipschitz network is typically a cascade of orthogonal layers and GroupSort activations~\citep{li2019preventing}, without skip-connections and normalization layers.
The lack of critical factors for training deep networks, including {\em skip-connections}~\citep{he2016deep,he2016identity}, {\em batch normalization}~\citep{ioffe2015batch,ioffe2017batch}, and {\em proper initialization}~\citep{glorot2010understanding,xiao2018dynamical}, is the main reason why the Lipschitz network have not scaled to deeper architectures.
Therefore, in this section, we explore skip-connections and initialization methods toward addressing this problem.

\vspace{0.2em}
{\bf Skip-connections.} 
Two general types of skip-connections are widely used in deep networks, one based on {\em addition} and another on {\em concatenation}. The addition-based connection is proposed in \ResNet~\citep{he2016deep}, and adopted in state-of-the-art architectures like SE-Net~\citep{hu2018squeeze} and EffcientNet~\citep{tan2019efficientnet}. The concatenation-based  connection is proposed in flow-based generative models~\citep{dinh2014nice,dinh2016density,kingma2018glow}, and adopted in reversible networks~\citep{gomez2017reversible,jacobsen2018revnet} as well as \ShuffleNet~\citep{zhang2018shufflenet,ma2018shufflenet}.
In what follows, we propose Lipschitz skip-connections based on these two different mechanisms, illustrated in \Cref{fig:residual-blocks} (in \Cref{app-sec:deep-networks}).

\begin{restatable}[Lipschitzness of residual blocks]{proposition}{PropResidual}
\label{prop:lipschitz-resnet}
Suppose $f^1$, $f^2$ are two $L$-Lipschitz functions (for residual and shortcut branches) and $\alpha \in [0, 1]$ is a learnable scalar, then an additive residual block $f$ defined in \Cref{eq:resnet-block} is $L$-Lipschitz.
\begin{equation}
\label{eq:resnet-block}
f(\signal{x}) \triangleq \alpha f^1(\signal{x}) + (1 - \alpha) f^2(\signal{x}),
\end{equation}
Alternatively, suppose $g^1$, $g^2$ are two $L$-Lipschitz functions and $\mymatrix{P}$ denotes channel permutation, then a concatenative residual block $g$ defined in \Cref{eq:shufflenet-block} is $L$-Lipschitz
\begin{equation}
\label{eq:shufflenet-block}
g(\signal{x}) \triangleq \mymatrix{P}\left[g^1(\signalSup{x}{1}); g^2(\signalSup{x}{2}) \right]
\end{equation}
where $[\cdot;~\cdot]$ denotes channel concatenation, and $\signal{x}$ is split into $\signal{x}_1$ and $\signal{x}_2$, i.e., $\signal{x} = [\signal{x}_1, \signal{x_2}].$ 
\end{restatable}

\vspace{0.2em}
{\bf Initialization.}
Proper initialization is crucial in training deep networks~\citep{glorot2010understanding,he2016deep}. In the context of orthogonal {\RNN}s, various methods are proposed to initialize orthogonal matrices, including the identical/permutation and torus initialization~\citep{henaff2016recurrent,helfrich2018orthogonal,lezcano2019cheap}.
However, initialization of orthogonal convolutional layers was not studied, and all previous approaches inherit the initialization from the underlying parameterization~
\citep{li2019preventing,trockman2021orthogonalizing}.
In \Cref{prop:paraunitary-initialization}, we show how to apply the initialization methods for orthogonal matrices to orthogonal convolutions (represented as in \Cref{eq:sublayer-spectral}). 

In the experiments, we will evaluate the impact of different choices of skip-connections and initialization methods to the performance of deep Lipschitz networks.
\section{Related Work}
\label{sec:related}

\textbf{Orthogonality in neural networks} are advocated in different contexts.
For RNNs to capture long-term dependence, orthogonal transitions are proposed to address the gradient vanishing/exploding problems~\citep{arjovsky2016unitary,wisdom2016full,mhammedi2017efficient,jing2017tunable,henaff2016recurrent,vorontsov2017orthogonality,helfrich2018orthogonal,maduranga2019complex}.
In parallel, orthogonal initialization/regularization~\citep{xiao2018dynamical,qi2020deep} is proposed to ensure signal propagation in training normalization-free deep networks.
Furthermore, orthogonal layers are used in Lipschitz networks, which ensures robustness against perturbations~\citep{cisse2017parseval,anil2019sorting,li2019preventing,trockman2021orthogonalizing}.
Lastly, \citet{jia2019orthogonal} suggests that orthogonal weights decrease the generalization error.

\textbf{Enforcing orthogonal constraints} has two typical families of approaches, one using projected or Riemannian gradient descent and another representing orthogonal matrices/convolutions with unconstrained parameters.
{\bf (1)} For {\em orthogonal matrices}, various types of parameterization are proposed, using Fourier matrix~\citep{arjovsky2016unitary,jing2017tunable}, Householder reflection~\citep{mhammedi2017efficient}, Cayley transform~\citep{helfrich2018orthogonal,maduranga2019complex}, and matrix exponential~\citep{lezcano2019cheap,lezcano2019trivializations}. 
Alternatively, \citep{anil2019sorting} proposes projected gradient descent via Björck's algorithm, and \citep{wisdom2016full} and \citep{vorontsov2017orthogonality} consider Riemannian gradient descent.
{\bf (2)} For {\em orthogonal convolutions}, early works enforce orthogonality of the flattened convolution kernel~\citep{jia2017improving,cisse2017parseval}, which, however, does not lead to the orthogonality of the original convolution.
Projected gradient descent via singular value clipping is proposed in \citep{sedghi2018singular}, which is expensive in practice.
Recent works adopt parameterization-based approaches, either using block convolutions~\citep{li2019preventing} or Cayley transform of convolutions~\citep{trockman2021orthogonalizing}.

\textbf{Paraunitary systems}
were extensively studied in filter banks and wavelets~\citep{vaidyanathan1993multirate,strang1996wavelets,lin1996theory}.
Classic theory shows that 1D-paraunitary systems are fully characterized by spectral factorization (see Chapter 14 of \citep{vaidyanathan1993multirate}), but not all MD-paraunitary systems admit a factorized form (see Chapter 8 of \citep{lin1996theory}).
While the complete characterization of MD-paraunitary systems is known in theory~\citep{venkataraman1995comparison,zhou2005multidimensional}, most practical constructions use separable paraunitary systems~\citep{lin1996theory} and special classes of non-separable paraunitary systems~\citep{hurley2012paraunitary}.
The equivalence between orthogonal convolutions and paraunitary systems thus opens the opportunities to apply these classic theories in orthogonal convolution designs.
\section{Experiments}
\label{sec:experiments}

In the experiments, we achieve three goals.
{\bf (1)} We demonstrate in \Cref{sub:exact-orthogonality} that our \SCFaclong (\SCFac) achieves precise orthogonality (up to machine-precision), resulting in more accurate orthogonal designs than previous ones~\citep{sedghi2018singular,li2019preventing,trockman2021orthogonalizing}.
{\bf (2)} Despite the differences in preciseness, we show in \Cref{sub:adversarial-robustness} that different realizations of paraunitary systems only have a minor impact on the adversarial robustness of Lipschitz networks.
{\bf (3)} Due to the versatility of our convolutional layers and architectures, in \Cref{sub:deep-networks}, we explore the best strategy to scale Lipschitz networks to wider/deeper architectures.
Training details are provided in \cref{app-sub2:training-details}.

\subsection{Exact Orthogonality}
\label{sub:exact-orthogonality}

We evaluate the orthogonality of our \SCFac layer verse previous designs, including \CayleyConv~\citep{trockman2021orthogonalizing}, \BCOP~\citep{li2019preventing}, \SVCM~\citep{sedghi2018singular}, \RKO~\citep{cisse2017parseval}, \OSSN~\citep{miyato2018spectral}. Our experiments are based on convolutional layers with $64$ input channels and $16 \times 16$ input size (typical in a neural network). These layers are made orthogonal using various approaches, and evaluated with Gaussian inputs. For our \SCFac layer, we initialize all orthogonal factors uniformly, and we use the built-in initialization for other approaches.
We evaluate the difference between $1$ and the ratio of the output norm to the input norm, and a layer is precisely orthogonal if the number is close to $0$.

\vspace{-0.5em}
\begin{table}[!htbp]
    \centering
    \caption{\small{{\bf (Left) Orthogonality evaluation of different designs for standard convolution.} 
    The number $\norm {\mathsf{Conv}(\signal{x})} / \norm{\signal{x}} - 1$ indicates the difference between the output and input norms of a layer. A layer is more precisely orthogonal if the number is closer to $0$. 
    As shown, our \SCFac achieves orders of magnitude more orthogonal on standard convolution.
    {\bf (Right) Orthogonality evaluation of our \SCFac design for various convolutions.}
    The numbers $\norm {\mathsf{Conv}(\signal{x})} / \norm{\signal{x}} - 1$ displayed are in the magnitude of $10^{-8}$.
    As shown, our \SCFac achieves machine epsilon orthogonality on variants of convolution.}}
    \vspace{-0.5em}
    \small
    \begin{minipage}{.4\textwidth}
    \label{tab:standard_conv}
    \renewcommand{\tabcolsep}{3pt}
    \centering
    \resizebox{.98\textwidth}{!}{
    \fbox{\begin{tabular}{l | c}
        \toprule
        Conv.  &  $\|\mathsf{Conv}(\signal{x})\|/\|\signal{x}\| - 1$ \\
        \midrule
        \SCFac (Ours)
        & $\mathbf{(+3.14 \pm 7.38) \times 10^{-8}}$
        \\
        \CayleyConv~\citep{trockman2021orthogonalizing}
        & $(+2.88 \pm 1.90) \times 10^{-4}$ 
        \\
        \BCOP~\citep{li2019preventing}
        &  $(+2.59 \pm 6.14) \times 10^{-3}$ 
        \\ \midrule
        \SVCM~\citep{sedghi2018singular}
        & $-0.429 \pm 3.31 \times 10^{-3}$ 
        \\
        \RKO~\citep{cisse2017parseval} 
        & $-0.666 \pm 1.74 \times 10^{-3}$ 
        \\
        \OSSN~\citep{miyato2018spectral}
        & $-0.422 \pm 3.44 \times 10^{-3}$ 
        \\
        \bottomrule
    \end{tabular}}
    }
    \end{minipage}
    \hfill
    \begin{minipage}{.58\textwidth}
    \centering
    \scriptsize
    \label{tab:various_conv}
    \setlength{\tabcolsep}{3pt}
    \resizebox{\textwidth}{!}{
    \fbox{\begin{tabular}{c c| c c c}
    \toprule
    \multicolumn{2}{c|}{\diagbox{Type}{Groups}} & 1 & 4 & 16 \\
    \midrule
    \multirow{3}{*}{$R$-Dilated} 
    & 1 & $+3.14 \pm 7.38$ & $+1.94 \pm 6.87$ & $+1.44 \pm 6.29$ \\
    & 2 & $+3.65 \pm 7.87$ & $+1.41 \pm 6.77$ & $+1.02 \pm 6.46$ \\
    & 4 & $+3.18 \pm 7.46$ & $+1.79 \pm 6.87$ & $+1.54 \pm 6.21$ \\
    \midrule
    \multirow{2}{*}{$\downarrow \!R$-Strided}
    & 2 & $-4.69 \pm 5.10$ & $+4.38 \pm 6.30$ & $+1.79 \pm 5.78$ \\
    & 4 & $+10.39\!\pm\!5.15$ & $+6.35 \pm 6.04$ & $+3.05 \pm 5.79$ \\
    \multirow{2}{*}{$\uparrow\!R$-Strided} 
    & 2 & $+3.67 \pm 7.96$ & $+1.38 \pm 6.70$ & $+1.43 \pm 6.23$ \\
    & 4 & $+3.86 \pm 7.09$ & $+1.12 \pm 6.81$ & N/A \\ 
    \bottomrule
    \end{tabular}}}
    \end{minipage}
\vspace{-0.5em}
\end{table}

{\bf (1) Standard convolution.}
We show in \Cref{tab:standard_conv} (Left) that our \SCFac is orders of magnitude more precise than all other approaches. The \SCFac layer is in fact exactly orthogonal up to machine epsilon, which is $2^{-24} \approx 5.96 \times 10^{-8}$ for $32$-bits floats. While \RKO and \OSSN are known not to be orthogonal, we surprisingly find that \SVCM is far from orthogonal due to its masking step.
{\bf (2) Convolutions variants.}
In \Cref{sec:framework}, we show that various orthogonal convolutions can be constructed using paraunitary systems. We verify our theory in \Cref{tab:standard_conv} (Right): our \SCFac layers are exactly orthogonal (up to machine precision) for various convolutions.

\subsection{Adversarial Robustness}
\label{sub:adversarial-robustness}

In this subsection, we evaluate the adversarial robustness of Lipschitz networks. Following the setup in \citep{trockman2021orthogonalizing}, we adopt KW-Large, ResNet9, WideResNet10-10 as the backbone architectures, and evaluate their robust accuracy on CIFAR-10 with different designs of orthogonal convolutions. We extensively perform a hyper-parameter search and choose the best hyper-parameters for each approach based on the robust accuracy. The details of the hyper-parameter search is in \Cref{app-sec:experiments}. We run each model with 5 different seeds and report the best accuracy. 

{\bf (1) Certified robustness.}
Following \citet{li2019preventing}, we use the raw images (without normalization) for network input to achieve the best certified accuracy. As shown in \Cref{tab:result-certified} (Top), different realizations of paraunitary systems,  \SCFac, \CayleyConv and \BCOP  achieve comparable performance --- \CayleyConv performs $<1\%$ better in clean accuracy, but the difference in robust accuracy are negligible. 
{\bf (2) Practical robustness.}
\citet{trockman2021orthogonalizing} shows that the certified accuracy is too conservative, and it is possible to increase the practical robustness (against PGD attacks) with a standard input normalization. Notice that the normalization increases the Lipschitz bound, thus lower the certified accuracy. Our experiments in \Cref{tab:result-certified} (Bottom) are based on ResNet9, WideResNet10-10~\citep{trockman2021orthogonalizing} and a deeper WideResNet22. For the shallow architectures (ResNet9, WideResNet10-10), our \SCFac, \CayleyConv, and \BCOP again achieve comparable performance --- \CayleyConv is slightly ahead in robust accuracy. \textbf{For the deeper architecture, our \SCFac has a clear advantage in both clean and robustness accuracy}, and the clean accuracy to only $5\%$ lower than a traditional \ResNet32 trained with batch normalization. Surprisingly, we find that \RKO also performs well in robust accuracy while not exactly orthogonal.
In summary, our experiments show that various paraunitary realizations provide different impacts on certified and practical robustness. While exact orthogonality provides tight Lipschitz bound, there is a trade-off between the exact orthogonality and the practical robustness (especially with the shallow architectures).

\begin{table}[!htbp]
\centering
\caption{\small{{\bf (Top) Certified robustness for plain convolutional networks} (without input normalization). We use KW-Large introduced by \citet{wong2018scaling}. The results for \RKO, \OSSN, and \SVCM are produced by \citet{trockman2021orthogonalizing}. {\bf (Bottom) Practical robustness for residual networks} (with input normalization). For 22 layers, the width of \SCFac is multiplied with 10, \CayleyConv with 6, and \BCOP and \RKO with 8. We are unable to scale \CayleyConv, \BCOP, and \RKO due to memory constraint. As shown, deeper architectures perform better than shallow ones for all orthogonal convolution types, and our \SCFac has a clear advantage.
}}
\vspace{-0.5em}
\renewcommand{\tabcolsep}{3pt}
\resizebox{0.55\textwidth}{!}{
\fbox{\begin{tabular}{c|c|c c c | c c c}
    \toprule
    & & \multicolumn{6}{c}{KW-Large} \\
    \cmidrule{3-8}
    $\epsilon$ & Test Acc. & \SCFac & \CayleyConv & \BCOP & \RKO  & \OSSN & \SVCM \\
    \midrule
    0 & Clean & 74.69 & 75.57 & 74.81 & 74.47 & 71.69 & 72.43 \\
    \midrule
    \multirow{2}{*}{$\frac{36}{255}$} & Certified & 58.68 & 59.03 & 58.83 & 57.50 & 55.71 & 52.11 \\
    & PGD & 67.72 & 67.78 & 67.47 & 68.32 & 65.13 & 66.43 \\
    \bottomrule
\end{tabular}}}
\resizebox{\textwidth}{!}{
\fbox{\begin{tabular}{c|c|cccc|cccc|cccc}
    \toprule
    & & \multicolumn{4}{c|}{ResNet9} & \multicolumn{4}{c|}{WideResNet10-10} & \multicolumn{4}{c}{WideResNet22-max} \\
    \cmidrule{3-14}
    $\epsilon$ & Test Acc. & \SCFac & \CayleyConv & \BCOP & \RKO & \SCFac & \CayleyConv & \BCOP & \RKO & \SCFac & \CayleyConv & \BCOP & \RKO \\
    \midrule
    0 & Clean & {82.19}  & \textbf{84.26} & 83.20 & 84.07 & {84.09} & 82.99 & 84.29 & \textbf{84.51} &  \textbf{87.82} & 85.85& 84.50 & 84.55 \\
    \midrule
    $\frac{36}{255}$ & PGD & 71.21 & 73.47  & 73.05 & \textbf{75.03} & 74.29 & {76.02} & 74.60 & \textbf{77.14} & \textbf{76.46} & {74.81} & 75.00 & 76.41  \\
    \bottomrule
\end{tabular}}}
\label{tab:result-certified}
\vspace{-0.5em}
\end{table}

\subsection{Scaling-up Deep Lipschitz Networks}
\label{sub:deep-networks}

All previously proposed Lipschitz networks~\citep{li2019preventing,trockman2021orthogonalizing} consider only shallow architectures (less than $10$ layers). 
In this subsection, we investigate various factors to scale Lipschitz networks to deeper architectures: skip-connection, depth-width, receptive field, and down-sampling.

{\bf (1) Skip-connections.}
Conventional wisdom suggests that skip-connections mainly address the gradient vanishing/exploding problem; thus, they are not needed for orthogonal networks. To understand their role, we perform an experiment that trains deep Lipschitz networks without skip-connection and with skip-connections based on addition/concatenation (see \Cref{sec:deep-networks}). As shown in \Cref{tab:result-skip-kernel} (left), the network with additive skip-connection substantially outperforms the other two, and the one without skip-connections performs the worst. Therefore, we empirically show that (additive) skip-connection is crucial in deep Lipschitz networks. 
{\bf (2) Depth and width.}
Exact orthogonality is criticized for harming the expressive power of neural networks.
We show that the decrease of expressive power can be alleviated by increasing the network depth/width. In \Cref{tab:result-certified} (Bottom) and \Cref{tab:depth-1} (\Cref{app-sec:experiments}), we observe that deeper/wider architectures increase both the clean and robust accuracy.
{\bf (3) Initialization methods.} 
We try different initialization methods, including identical, permutation, uniform, and torus~\citep{henaff2016recurrent,helfrich2018orthogonal}. We find that identical initialization works the best for deep Lipschitz networks ($>10$ layers), while all methods perform similarly in shallow networks as shown in \Cref{tab:result-init} (\Cref{app-sec:experiments}).
{\bf (4) Receptive field and down-sampling.}
Previous works~\citep{li2019preventing,trockman2021orthogonalizing} use larger kernel size and no stride for Lipschitz networks.
In \Cref{tab:result-skip-kernel} (Right), we perform a study on the effects of kernel/dilation size and down-sampling types for the orthogonal convolutions.
We find that an average pooling as down-sampling consistently outperforms strided convolutions. Furthermore, a larger kernel size helps to boost the performance.
{\bf (5) Run-time and memory comparison.}
We find that previously proposed orthogonal convolutions such as \CayleyConv, \BCOP, and \RKO require more GPU memory and computation time than \SCFac. Therefore, we could not to scale them due to memory constraints (for 22 and 32 layers using Tesla V100 32G). 
In order to scale up Lipschitz networks, economical implementation of orthogonal convolution is crucial. As shown in \Cref{fig:comp_time_memory}, for deep and wide architectures, our \SCFac is the most computationally and memory efficient method and the only method that scales to a width increase of 10 on WideResNet22.
Missing numbers in \Cref{fig:comp_time_memory} and \Cref{tab:depth-1} (\Cref{app-sec:experiments}) are due to the large memory requirement.


    

\begin{table}[!tbp]
\centering
\caption{\small{{\bf (Left) Comparisons of various skip connection types} on WideResNet22-10 (kernel size 5). 
{\bf (Right) Comparisons of various receptive field and down-sampling types} on WideResNet10-10.} The symbols \ding{51}, \ding{53} indicate whether average pooling or strided convolutions are used for down-sampling. The kernel size for  strided convolutions is equal to the stride for ``slim'', which is tripled for ``wide''.}
\vspace{-0.5em}
    \hfill
    \begin{minipage}{.37\textwidth}
    \centering
        \resizebox{\textwidth}{!}{
        \fbox{\begin{tabular}{l|cc}
            \toprule
            \multirow{2}{*}{Skip type} & \multicolumn{2}{c}{Test Acc.} \\
            \cmidrule{2-3}
             & Clean & PGD \\
            \midrule
             ConvNet (w/o skip) & 69.59 & 59.22  \\
             ShuffleNet (concat) & 75.21 & 66.00 \\
             ResNet (add) & \textbf{87.82} & \textbf{76.46} \\
            \bottomrule
        \end{tabular}}}
    \end{minipage}
    \hfill
    \begin{minipage}{.6\textwidth}
    \centering
      \small
        \resizebox{0.9\textwidth}{!}{
        \fbox{\begin{tabular}{cc|cc|cc}
            \toprule
            \multicolumn{2}{c|}{Receptive Field} & \multicolumn{2}{c|}{Down-Sampling} & \multicolumn{2}{c}{Test Acc.} \\
            \midrule
            Kernel & Dilation & Pool & Stride & Clean & PGD \\
            \midrule
            \multirow{1}{*}{3} & \multirow{1}{*}{1} & \multirow{1}{*}{\ding{53}} & \multirow{1}{*}{slim} & 80.70 & 68.81 \\ 
            \multirow{1}{*}{3} & \multirow{1}{*}{1} & \multirow{1}{*}{\ding{53}} & \multirow{1}{*}{wide} & {82.36} & {70.36} \\ 
            \midrule
            \multirow{1}{*}{3} & \multirow{1}{*}{1} & \multirow{1}{*}{\ding{51}} & \multirow{1}{*}{\ding{53}} & {84.54} & {71.71} \\ 
            \multirow{1}{*}{3} & \multirow{1}{*}{2} & \multirow{1}{*}{\ding{51}} & \multirow{1}{*}{\ding{53}} & 81.53 & 70.07 \\
            \midrule
            \multirow{1}{*}{5} & \multirow{1}{*}{1} & \multirow{1}{*}{\ding{51}} & \multirow{1}{*}{\ding{53}} & \textbf{84.09} & \textbf{74.29} \\
            \multirow{1}{*}{5} & \multirow{1}{*}{2} & \multirow{1}{*}{\ding{51}} & \multirow{1}{*}{\ding{53}} & 81.28 & 70.58 \\
            \bottomrule
        \end{tabular}}}
    \end{minipage} 
    \hfill
\label{tab:result-skip-kernel}
\vspace{-1em}
\end{table}

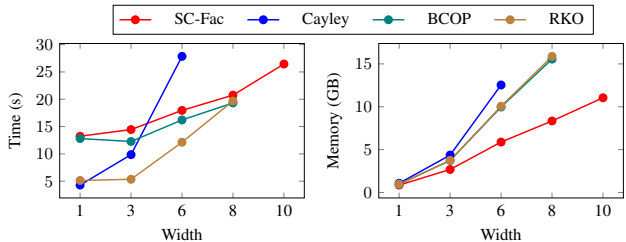
\begin{wrapfigure}[12]{r}{0.6\textwidth}
    \vspace{-1.3em}
    \centering
    \resizebox{\textwidth}{!}{
    \begin{tikzpicture}
    \pgfplotsset{footnotesize,samples=10}
        \begin{groupplot}[group style = {group size = 2 by 1, horizontal sep = 40pt}, width=6.2cm, height=4.5cm]
            \nextgroupplot[legend style = {column sep = 10pt, legend columns = -1, font=\small, at={(1.2,1.25)}, anchor=north},
                tick label style={font=\footnotesize},
                label style={font=\footnotesize},
                symbolic x coords={1,3,6,8,10},
                xtick=data,
                xlabel={Width},
                ylabel={Time (s)}]
                 \addplot[red, thick, mark=*]
                    coordinates{(1, 13.23) (3, 14.45) (6, 17.96) (8, 20.75) (10, 26.45)};
                \addplot[blue, thick, mark=*]
                    coordinates{(1, 4.29) (3, 9.86) (6, 27.85) };
                \addplot[teal, thick, mark=*]
                    coordinates{(1, 12.81) (3, 12.27) (6, 16.21) (8, 19.32)};
                \addplot[brown, thick, mark=*]
                    coordinates{(1, 5.13) (3, 5.35) (6, 12.13) (8, 19.68)};
                \legend{SC-Fac, Cayley, BCOP, RKO}
            \nextgroupplot[
                tick label style={font=\footnotesize},
                label style={font=\footnotesize},
                symbolic x coords={1,3,6,8,10},
                xtick=data,
                xlabel={Width},
                ylabel={Memory (GB)}]
                \addplot[red, thick, mark=*]
                    coordinates{(1, 0.86) (3, 2.68) (6, 5.88) (8, 8.34) (10, 11.05)};
                \addplot[blue, thick, mark=*]
                    coordinates{(1, 1.07) (3, 4.36) (6, 12.54) };
                \addplot[teal, thick, mark=*]
                    coordinates{(1, 0.96) (3, 3.73) (6, 9.97) (8, 15.56)};
                \addplot[brown, thick, mark=*]
                    coordinates{(1, 0.96) (3, 3.74) (6, 10.06) (8, 15.88)};
        \end{groupplot}
    \end{tikzpicture}
    }
    \vspace{-2em}
    \caption{\small {\bf Run-time and memory comparison} using WideResNet22 on Tesla V100 32G. x-axis indicates the width factor (channels $=$ base\_channels $\times$ factor). 
    Our \SCFac is the most computationally and memory efficient for wide architectures and is the only method that scales to width factor to 10 on WideResNet22.}
    \label{fig:comp_time_memory}
\end{wrapfigure}

In summary, we find out the additive skip-connection is still essential for learning deep orthogonal networks. Due to the orthogonal constraints, it is helpful to increase the depth/width of the network. However, this significantly increases the memory requirement; thus, a cheap implementation (like our \SCFac) is desirable. Finally, we find that a larger kernel size and down-sampling based on average pooling is helpful, unlike the standard practices in deep networks.
\section{Conclusion}
\label{sec:conclusion}

In this paper, we presented a framework for orthogonal convolutions based on paraunitary systems. Specifically, we establish the equivalence between orthogonal convolutions in the spatial domain and paraunitary systems in the spectral domain. Therefore, any method that enforces orthogonality in convolutions is implicitly designing paraunitary systems. We further show that the orthogonality for variants of convolution (strided, dilated, group convolutions) is also fully characterized by paraunitary systems. As a result, paraunitary systems are all we need to ensure orthogonality for diverse types of convolutions. 
Based on the complete factorization of paraunitary systems, we develop the first exact and complete design for various convolution layers, ensuring both exact orthogonality and high expressiveness. Our versatile design allows us to perform a systematic study of the design principles for orthogonal convolutional networks. As a result, we show that deep orthogonal networks can scale up to deeper architectures while maintaining a Lipschitz upper bound.

\clearpage
\bibliography{\bibhome/supp_bib}
\bibliographystyle{unsrtnat}

\clearpage
\section*{Checklist}

\begin{enumerate}

\item For all authors...
\begin{enumerate}
  \item Do the main claims made in the abstract and introduction accurately reflect the paper's contributions and scope?
    \answerYes{See \Cref{sec:paraunitary,sec:framework,sec:deep-networks}.}
  \item Did you describe the limitations of your work?
    \answerYes{See \Cref{sec:paraunitary} and \Cref{sec:experiments}}
  \item Did you discuss any potential negative societal impacts of your work?
    \answerYes{See \Cref{app-sec:impact-statement}}
  \item Have you read the ethics review guidelines and ensured that your paper conforms to them?
    \answerYes{}
\end{enumerate}

\item If you are including theoretical results...
\begin{enumerate}
  \item Did you state the full set of assumptions of all theoretical results?
    \answerYes{See \Cref{app-sec:paraunitary,app-sec:framework,app-sec:deep-networks}
    }
	\item Did you include complete proofs of all theoretical results?
    \answerYes{\Cref{app-sec:paraunitary,app-sec:framework,app-sec:deep-networks}
    }
\end{enumerate}

\item If you ran experiments...
\begin{enumerate}
  \item Did you include the code, data, and instructions needed to reproduce the main experimental results (either in the supplemental material or as a URL)?
    \answerYes{Provided in \Cref{app-sec:experiments}}
  \item Did you specify all the training details (e.g., data splits, hyperparameters, how they were chosen)?
    \answerYes{Provided in \Cref{app-sec:experiments}.}
	\item Did you report error bars (e.g., with respect to the random seed after running experiments multiple times)?
    \answerYes{See \Cref{sec:experiments}}
	\item Did you include the total amount of compute and the type of resources used (e.g., type of GPUs, internal cluster, or cloud provider)?
    \answerYes{See \Cref{sec:experiments}}
\end{enumerate}

\item If you are using existing assets (e.g., code, data, models) or curating/releasing new assets...
\begin{enumerate}
  \item If your work uses existing assets, did you cite the creators?
    \answerYes{See \Cref{sec:experiments}}
  \item Did you mention the license of the assets?
    \answerNA{}
  \item Did you include any new assets either in the supplemental material or as a URL?
    \answerNA{}
  \item Did you discuss whether and how consent was obtained from people whose data you're using/curating?
   \answerNA{}
  \item Did you discuss whether the data you are using/curating contains personally identifiable information or offensive content?
    \answerNA{}
\end{enumerate}

\item If you used crowdsourcing or conducted research with human subjects...
\begin{enumerate}
  \item Did you include the full text of instructions given to participants and screenshots, if applicable?
    \answerNA{}
  \item Did you describe any potential participant risks, with links to Institutional Review Board (IRB) approvals, if applicable?
    \answerNA{}
  \item Did you include the estimated hourly wage paid to participants and the total amount spent on participant compensation?
    \answerNA{}
\end{enumerate}

\end{enumerate}

\clearpage
\appendix

{\begin{center}{\bf \Large Appendix: \mytitle}\end{center}}

\textbf{Notations.} 
We use non-bold letters for scalars (e.g., $x$) and bold letters for vectors/matrices (e.g., $\signal{x}$). We denote sequences in the spatial domain using lower-case letters (e.g., $\signalX{x}{n}$) and their spectral representations using upper-case letters (e.g., $\SignalX{X}{z}$, $\SignalX{X}{e^{\im \omega}}$). For a positive integer, say $R \in Z^{+}$, we abbreviate the set $\{0, 1, \cdots, R-1 \}$ as $\range{R}$, and whenever possible, we use its lower-case letter, say $r \in \range{R}$, as the corresponding iterator. 

\vspace{0.3em}
\textbf{Assumptions.}
For simplicity, we assume all sequences are $\mathcal{L}^2$ with range $\Z$ (a sequence $\{\signalX{x}{n}, n \in \Z\}$ is $\mathcal{L}^2$ if $\sum_{n \in \Z} \norm{\signalX{x}{n}}^2 < \infty$). Such assumption is common in the literature, which avoids boundary conditions in signal analysis. To deal with periodic sequences (finite sequences with circular padding), people can either adopt the Dirac function in the spectral domain or use discrete Fourier transform. In our implementation, however, we address the boundary condition case by case for each convolution type, with which we achieve exact orthogonality in the experiments (\Cref{sub:exact-orthogonality}). 

\section{Orthogonal Convolutions via Paraunitary Systems}
\label{app-sec:paraunitary}

In this section, we prove the convolution theorem and Parseval's theorem in the context of standard convolutional layers. 
Subsequently, we prove the paraunitary theorem which establishes the equivalence between orthogonal convolutional layers and paraunitary systems.

\subsection{Spectral Analysis of Standard Convolution Layers}
\label{app-sub:paraunitary-spectral}

\textbf{Standard convolutional layers} are the default building blocks for convolutional neural networks. One such layer consists of a filter bank with $T \times S$ filters $\filter{h} = \{\filterSubX{h}{ts}{n}, n \in \Z\}_{t \in [T], s \in [S]}$, where $S$, $T$ are the number of input and output channels respectively. The layer maps an $S$-channel input $\signal{x} = \{\signalSubX{x}{s}{i}, i \in \Z\}_{s \in [S]}$ to a $T$-channel output $\signal{y} = \{\signalSubX{y}{t}{i}, i \in \Z\}_{t \in [T]}$ according to
\begin{equation}
\label{eq:conv1d}
\signalSubX{y}{t}{i} 
= \sum_{s \in [S]} \sum_{n \in \Z} \filterSubX{h}{ts}{n} \signalSubX{x}{s}{i-n},
\end{equation}
where $i$ indexes the output location to be computed, and $n$ indexes the filter coefficients. Alternatively, we can rewrite \Cref{eq:conv1d} in matrix-vector form as
\begin{equation}
\label{eq:conv1d-matrix}
\signalX{y}{i} = \sum_{n \in \Z} \filterX{h}{n} \signalX{x}{i - n},
\end{equation}
where each $\filterX{h}{n} \in \R^{T \times S}$ is a matrix, and $\signalX{x}{n} \in \R^{S}$ (or $\signalX{y}{n} \in \R^{T}$) is a vector.

Notice that in \Cref{eq:conv1d-matrix} we group entries from {\em all channels} into a vector/matrix (e.g., from $\{\signalSubX{x}{0}{n}\}_{s \in \range{S}}$ to $\signalX{x}{n}$), different from a common notation that groups entries from {\em all locations} into a vector/matrix (e.g., from $\{\signalSubX{x}{s}{n}, n \in \Z\}$ into $\signal{x}_s$). In matrix-vector form, a convolutional layer is a convolution between two sequence of matrices/vectors.

Let us first define the {\bf \ZTLong} and {\bf Fourier transforms} in \Cref{def:z-transform} before proving the convolution theorem (\Cref{thm:conv-theorem}).

\begin{definition}[\ZTLong and Fourier transforms]
\label{def:z-transform}
For a sequence (of scalars, vectors, or matrices) $\{\signal{x}[n], h \in \Z\}$, its \ZTLong $\SignalX{X}{z}$ is defined as
\begin{equation}
\label{eq:z-transform}
\FilterX{X}{z} = \sum_{n \in \Z} \filterX{x}{n} z^{-n},
\end{equation}
where $z \in \C$ is a complex number such that the infinite sum converges. If $z$ is restricted to the unit circle $z = e^{\im \omega}$ (i.e., $|z| = 1$), the \ZTlong $\SignalX{H}{z}$ reduces to a \DTFTLong (\DTFT) $\signalX{X}{e^{\im \omega}}$. If $\omega$ is further restricted to a finite set $\omega \in \{2\pi k / N, k \in \range{N}\}$, the \DTFT $\SignalX{X}{e^{\im \omega}}$ reduces to an $N$-points \DFTlong (\DFT) $\SignalX{X}{e^{\im 2\pi k/N}}$.
\end{definition}

The celebrated {\bf convolution theorem} states that the convolution in spatial domain (\Cref{eq:conv1d-matrix}) leads to a multiplication in the spectral domain, i.e., $\SignalX{Y}{z} = \FilterX{H}{z} \SignalX{X}{z}, \forall z \in \C$.

\begin{proof}[Proof of \Cref{thm:conv-theorem}]
The proof follows directly from the definitions of standard convolution (\Cref{eq:conv1d-matrix}) and \ZTLong (\Cref{eq:z-transform}).
\begin{align}
\SignalX{Y}{z} 
& = \sum_{i \in \Z} \signalX{y}{i} z^{-i} 
\label{proof:conv-theorem-step-1} \\
& = \sum_{i \in \Z} \left( \sum_{n \in \Z} \signalX{h}{n} \signalX{x}{i - n} \right) z^{-i} 
\label{proof:conv-theorem-step-2} \\
& = \sum_{n \in \Z} \signalX{h}{n} z^{-n} \left( \sum_{i \in \Z} \signalX{x}{i - n} z^{-(i - n)} \right) 
\label{proof:conv-theorem-step-3} \\
& = \left(\sum_{n \in \Z} \signalX{h}{n} z^{-n} \right) \left(\sum_{k \in \Z} \signalX{x}{k} z^{-k} \right) 
\label{proof:conv-theorem-step-4} \\
& = \SignalX{H}{z} \SignalX{X}{z},
\label{proof:conv-theorem-step-5}
\end{align}
where \Cref{proof:conv-theorem-step-1,proof:conv-theorem-step-5} use the definition of \ZTLong, \Cref{proof:conv-theorem-step-2} uses the definition of convolution, and \Cref{proof:conv-theorem-step-4} makes a change of variable $k = i - n$.
\end{proof}

Next, we introduce the concepts of {\bf inner product} and {\bf Frobenius norm} for sequences. We then prove {\bf Parseval's theorem}, which allows us to compute the sequence norm in the spectral domain.

\begin{definition}[Inner product and norm for sequences]
\label{def:sequence-inner-prodct-and-norm}
Given two sequences $\signal{x} = \{\signalX{x}{n}, n \in \Z\}$ and $\signal{y} = \{\signalX{y}{n}, n \in \Z\}$ with $\signalX{x}{n}, \signalX{y}{n}$ having the same dimension for all $n$, the inner product of these two sequences is defined as
\begin{equation}
\label{eq:sequence-inner-product}
\Product{\signal{x}}{\signal{y}} \triangleq
\sum_{n \in \Z} \Product{\signalX{x}{n}}{\signalX{y}{n}}
\end{equation}
where $\Product{\signalX{x}{n}}{\signalX{y}{n}}$ denotes the Frobenius inner product between $\signalX{x}{n}$ and $\signalX{y}{n}$. Subsequently, we can define the Frobenius norm of a sequence using inner product as
\begin{equation}
\label{eq:sequence-norm}
\left\|\signal{x}\right\| \triangleq \sqrt{\left< \signal{x}, \signal{x} \right>}
\end{equation}
\end{definition}

\begin{theorem}[Parsavel's theorem]
\label{thm:parseval-theorem}
Given a sequence $\{\signalX{x}{n}, n \in \Z\}$, its sequence norm $\norm{\signal{x}}$ can be computed by $\FilterX{X}{e^{\im \omega}}$ in the spectral domain as
\begin{equation}
\label{eq:parseval-theorem}
\norm{\signal{x}}^2 
= \sum_{n \in \Z} \norm{\signalX{x}{n}}^2
= \frac{1}{2\pi} \int_{-\pi}^{\pi} \norm{\SignalX{X}{e^{j\omega}}}^2 d\omega,
\end{equation}
where $\norm{\FilterX{X}{e^{\im \omega}}}^2 = \FilterX{X}{e^{\im \omega}}^{\dagger} \FilterX{X}{e^{\im \omega}}$ is an inner product between two identical complex arrays.
\end{theorem}

\begin{proof}[Proof of \Cref{thm:parseval-theorem}]
The theorem follows from the definitions of convolution and \DTFT.
\begin{align}
\frac{1}{2\pi} \int_{-\pi}^{\pi} 
\Norm{\SignalX{X}{e^{j\omega}}}^2 d\omega
& = \frac{1}{2\pi} \int_{-\pi}^{\pi} \Product{\SignalX{X}{e^{j\omega}}}{ \SignalX{X}{e^{j\omega}}} d \omega
\\
& = \frac{1}{2\pi} \int_{-\pi}^{\pi} 
\Product{\sum_{n \in \Z} \signalX{x}{n} e^{-\im \omega n}}{\sum_{m \in \Z} \signalX{x}{m} e^{-\im \omega m}} d\omega 
\\
& = \sum_{n \in \Z} \sum_{m \in \Z} \Product{\signalX{x}{n}}{\signalX{x}{m}} \frac{1}{2\pi} \int_{-\pi}^{\pi}  e^{-\im\omega(m - n)} d\omega 
\label{proof:parseval-theorem-step-1} \\
& = \sum_{n \in \Z} \sum_{m \in \Z} \Product{\signalX{x}{n}}{\signalX{x}{m}} \mathbbm{1}_{m = n}
\label{proof:parseval-theorem-step-2} \\
& = \sum_{n \in \Z} \left<\signalX{x}{n}, \signalX{x}{n} \right>
= \sum_{n \in \Z} \Norm{\signalX{x}{n}}^2,
\end{align}
where \Cref{proof:parseval-theorem-step-1} is due to the bi-linearity of inner products, and \Cref{proof:parseval-theorem-step-2} makes uses of the fact that $\int_{-\pi}^{\pi} e^{-\im\omega k} d\omega = 0$ for $k \neq 0$ and $\int_{-\pi}^{\pi} e^{-\im\omega k} d\omega = \int_{-\pi}^{\pi} d\omega = 2\pi$ for $k = 0$.
\end{proof}

\subsection{Equivalence between Orthogonal Convolutions and Paraunitary Systems}
\label{app-sub:paraunitary-equivalence}

With the sequence norm introduced earlier, we formally define orthogonality for convolutional layers.

\begin{definition}[Orthogonal convolutional layer]
\label{def:orthogonal-conv1d}
A convolution layer is orthogonal if the input norm $\norm{\signal{x}}$ is equal to the output norm $\norm{\signal{y}}$ for arbitrary input $\signal{x}$, that is
\begin{equation}
\label{eq:orthogonal-conv1d}
\|\signal{y}\|
\triangleq \sqrt{\sum_{n \in \Z} \Norm{\signalX{y}{n}}^2 } 
= \sqrt{\sum_{n \in \Z} \Norm{\signalX{x}{n}}^2}
\triangleq \|\signal{x}\|
\end{equation}
where $\norm{\signal{x}}$ (or $\norm{\signal{y}}$) is defined as the squared root of $\sum_{n \in \Z} \norm{\signalX{x}{n}}^2$ (or $\sum_{n \in \Z} \norm{\signalX{y}{n}}^2$).
\end{definition}

This definition of orthogonality not only applies to standard convolutions in \Cref{eq:conv1d-matrix} but also variants of convolutions in \Cref{app-sub:framework-construction}. In this section, however, we first establish the equivalence between {\em orthogonality for standard convolutions} and {\em paraunitary systems}.

\begin{theorem}[Paraunitary theorem]
\label{thm:paraunitary-conv1d}
A standard convolutional layer (in \Cref{eq:conv1d-matrix})
is orthogonal (by \Cref{def:orthogonal-conv1d}) 
if and only if its transfer matrix $\FilterX{H}{z}$ is paraunitary, i.e.,
\begin{equation}
\label{eq:paraunitary}
\adjoint{\FilterX{H}{z}} \FilterX{H}{z} = \identity, ~ \forall |z| = 1 \Longleftrightarrow \adjoint{\FilterX{H}{e^{\im \omega}}} \FilterX{H}{e^{\im \omega}} = \identity, ~ \forall \omega \in \R.
\end{equation}
In other words, the transfer matrix $\FilterX{H}{e^{\im \omega}}$ is unitary for all frequencies $\omega$.
\end{theorem}

\begin{proof}[Proof of \Cref{thm:paraunitary-conv1d}]
We first prove that a convolutional layer is orthogonal if its transfer matrix $\SignalX{H}{z}$ is paraunitary.
\begin{align}
\left\|\signal{y}\right\|^2 
& = \frac{1}{2\pi} \int_{-\pi}^{\pi} 
\Norm{\SignalX{Y}{e^{j\omega}}}^2 d\omega
\label{proof:paraunitary-step-1} \\
& = \frac{1}{2\pi} \int_{-\pi}^{\pi} 
\Norm{\FilterX{H}{e^{j\omega}} \SignalX{X}{e^{j\omega}}}^2 d\omega 
\label{proof:paraunitary-step-2} \\
& = \frac{1}{2\pi} \int_{-\pi}^{\pi} 
\Norm{\SignalX{X}{e^{j\omega}}}^2 d\omega 
\label{proof:paraunitary-step-3} \\
& = \Norm{\signal{x}}^2
\label{proof:paraunitary-step-4}
\end{align}
where \Cref{proof:paraunitary-step-1,proof:paraunitary-step-4} are due to Parseval's theorem (\Cref{thm:parseval-theorem}), \Cref{proof:paraunitary-step-1,proof:paraunitary-step-2} follows from the convolution theorem (\Cref{thm:conv-theorem}), and \Cref{proof:paraunitary-step-3} utilizes that $\FilterX{H}{e^{j\omega}}$ is unitary for all frequency $\omega$ (thus $\norm{\FilterX{H}{e^{j\omega}} \SignalX{X}{e^{j\omega}}} = \norm{\SignalX{X}{e^{j\omega}}}$ for any $\SignalX{X}{e^{j\omega}}$).

The `only if' part also holds in practice. We here prove for the special case of periodic inputs (e.g., finite inputs with circular padding). Suppose for contradiction there exists a frequency $\omega$ and $\FilterX{H}{e^{\im \omega}} $ is not unitary. Since $\FilterX{H}{e^{\im \omega}}$ is continuous (as a summation of complex sinusoids), there exists two integers $N$, $k$, such that such that $\omega \approx 2k\pi/N$ and $\FilterX{H}{e^{\im 2\pi k / N}}$ is also not unitary. As a result, there exist a complex vector $\myvector{u}$ such that $\myvector{v} = \SignalX{H}{e^{\im 2 \pi k / N}} \myvector{u}$ while $\|\myvector{v}\| \neq \|\myvector{u}\|$. Therefore, we can construct two periodic sequences $\{\signalX{x}{n}, n \in \range{N}\}$ and $\{\signalX{y}{n}, n \in \range{N}\}$ such that
\begin{equation}
\signalX{x}{n} = \myvector{u} e^{\im 2\pi n k / N} \implies \signalX{y}{n} = \myvector{v} e^{\im 2\pi n k / N}.
\end{equation}
Now the input norm $\|\signal{x}\| = \sqrt{\sum_{n \in \range{N}} \|\signalX{x}{n}\|^2} = \sqrt{N} \|\myvector{u}\|$ is not equal to the output norm $\|\signal{y}\| = \sqrt{\sum_{n \in \range{N}}  \|\signalX{y}{n}\|^2} = \sqrt{N} \|\myvector{v}\|$, i.e., the layer is not orthogonal, which leads to a contradiction.
\end{proof}
\section{A Paraunitary Framework for Orthogonal Convolutional Layers}
\label{app-sec:framework}

\begin{figure}[!htbp]
    \centering
    \includegraphics[trim={0.5cm 0.7cm 0.6cm 0.5cm},clip,width=\textwidth]{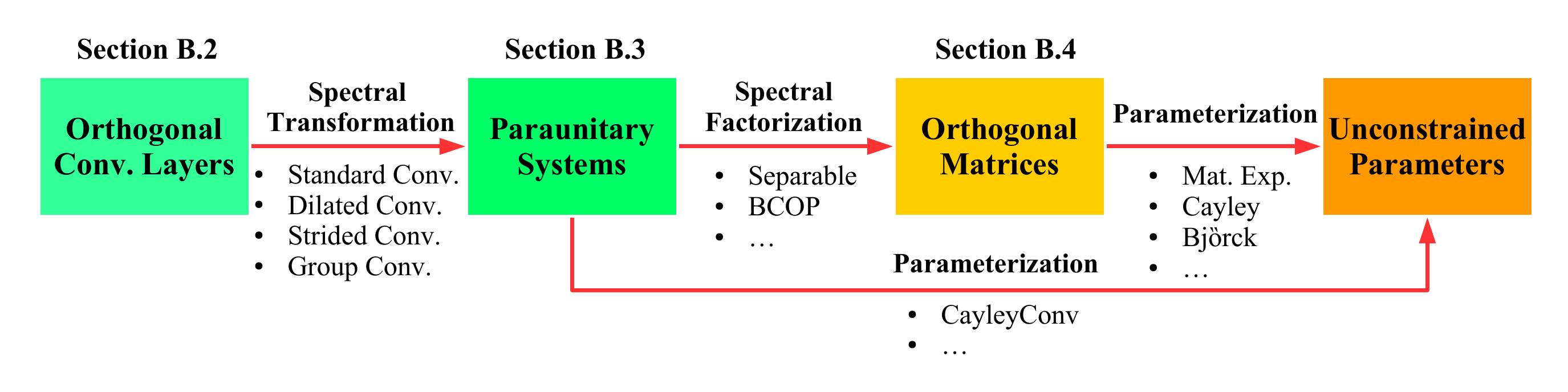}
\caption{
{\bf A framework for designing orthogonal convolutional layers.}
In \Cref{app-sub:framework-unification}, we unify variants of orthogonal convolutional layers in the spectral domain and show that their designs reduce to constructing paraunitary systems.
In \Cref{app-sub:framework-construction}, we show that a paraunitary system can be constructed with different approaches: our approach and \BCOP~\citep{li2019preventing} represent the paraunitary using orthogonal matrices, while \CayleyConv~\citep{trockman2021orthogonalizing} directly parameterizes it using unconstrained parameters.
In \Cref{app-sub:framework-parameterization}, we investigate various parameterizations for orthogonal matrices, including matrix exponential, Cayley transform, and Björck orthogonalization.
}
\label{fig:framework}
\end{figure}

\subsection{Multi-resolution Analysis}
\label{app-sub:multi-resolution}

Multi-resolution operations are essential in various convolutional layers, in particular strided and dilated convolutions. In order to define and analyze these convolutions rigorously, we first review the concepts of {\em up-sampling}, {\em down-sampling}, and {\em polyphase components}. 

\vspace{0.2em}
\textbf{(1) Up-sampling.}
Given a sequence (of scalars, vectors, matrices) $\signal{x}\!=\! \{\signalX{x}{n}, n \in \Z\}$, its {\em up-sampled sequence} $\signalSup{x}{\uparrow R} = \{\signalSupX{x}{\uparrow R}{n}, n \in \Z\}$ is defined as
\begin{equation}
\label{eq:up-sampling}
\signalSupX{x}{\uparrow R}{n} \triangleq 
\begin{cases}
\signalX{x}{n / R} & n \equiv 0 \!\!\!\!\! \pmod R \\
0 & \mathrm{otherwise}
\end{cases},
\end{equation}
where $R \in \Z^{+}$ is the up-sampling rate. Accordingly, we denote the \ZTLong of $\signalSup{x}{\uparrow R}$ as
\begin{equation}
\label{eq:up-sampling-z}
\SignalSupX{X}{\uparrow R}{z} = \sum_{n \in \Z} \signalSupX{x}{\uparrow R}{n} z^{-n}.
\end{equation}

The following proposition shows that $\SignalSupX{X}{\uparrow R}{z}$ is easily computed from $\SignalX{X}{z}$.

\begin{proposition}[Z-transform of up-sampled sequence]
\label{prop:up-sampling-z}
Given a sequence $\signal{x}$ and its up-sampled sequence $\signalSup{x}{\uparrow R}$, their {\ZTLong}s $\SignalX{X}{z}$ and $\SignalSupX{X}{\uparrow R}{z}$ are related by 
\begin{equation}
\label{eq:up-sampling-relation}
\SignalSupX{X}{\uparrow R}{z} = \SignalX{X}{z^R}.
\end{equation}
\end{proposition}

\begin{proof}[Proof of \Cref{prop:up-sampling-z}]
The proof makes direct use of the definition of \ZTLong.
\begin{align}
\SignalSupX{X}{\uparrow R}{z} 
& = \sum_{n \in \Z} \signalSupX{x}{\uparrow R}{n} z^{-n} 
\\
& = \sum_{m \in \Z} \signalSupX{x}{\uparrow R}{m R} z^{-m R} 
\label{proof:up-sampling-z-step} \\
& = \sum_{m \in \Z} \signalX{x}{m} (z^R)^{-m} = \SignalX{X}{z^{R}},
\end{align}
where \Cref{proof:up-sampling-z-step} makes a change of variables $m = n / R$ since $\signalSupX{x}{\uparrow R}{n} = 0, \forall n \neq m R$.
\end{proof}

\vspace{0.2em}
\textbf{(2) Down-sampling and polyphase components.} 
Different from the up-sampled sequence, there exist multiple down-sampled sequences, depending on the {\em phase} of down-sampling.
These sequences are named as the {\em polyphase components}. Specifically, given a sequence (of scalars, vectors, or matrices) $\signal{x} = \{\signalX{x}{n}, n \in \Z\}$,
its $r^{\mathrm{th}}$ polyphase component $\signalSup{x}{\polyphase{r}{R}} = \{\signalSupX{x}{\polyphase{r}{R}}{n}, n \in \Z\}$ is defined as
\begin{equation}
\label{eq:polyphase-component}
\signalSupX{x}{\polyphase{r}{R}}{n} \triangleq \signalX{x}{n R + r},
\end{equation}
where $R \in \Z^{+}$ is the down-sampling rate.
And we denote the \ZTLong of $\signalSup{x}{\polyphase{r}{R}}$ as
\begin{equation}
\SignalSupX{X}{\polyphase{r}{R}}{z} = \sum_{n \in \Z} \signalSupX{x}{\polyphase{r}{R}}{n} z^{-n},
\end{equation}
where $R$ is the down-sampling rate. Note that $r \in \Z$ is an arbitrary integer, which does not necessarily take values from $\range{R}$. In fact, we have $\signalSupX{x}{\polyphase{(r + k R)}{R}}{n} = \signalSupX{x}{\polyphase{r}{R}}{n + k}$ and  $\SignalSupX{X}{\polyphase{r + k R}{R}}{z} = z^{k} \SignalSupX{X}{\polyphase{r}{R}}{z}$. 
In the following proposition, 
we establish the relation between $\SignalX{H}{z}$ and $\{\SignalSupX{H}{\polyphase{r}{R}}{z}\}_{r \in \range{R}}$.

\begin{proposition}[Polyphase decomposition]
\label{prop:polyphase-decomposition}
Given a sequence $\signal{x}$ and its polyphase components $\signalSup{x}{\polyphase{r}{R}}$'s, the \ZTLong $\SignalX{X}{z}$ can be represented by $\{ \SignalSupX{X}{\polyphase{r}{R}}{z} \}_{r \in \range{R}}$ as 
\begin{equation}
\label{eq:polyphase-decomposition}
\SignalX{X}{z} = \sum_{r \in [R]}
\SignalSupX{X}{\polyphase{r}{R}}{z^R} z^{-r}.
\end{equation}
\end{proposition}

\begin{proof}[Proof of \Cref{prop:polyphase-decomposition}]
We begin with $\SignalX{X}{z}$, and try to decompose it into its polyphase components $\{\SignalSupX{X}{\polyphase{r}{R}}{z}\}_{r \in \range{R}}$.
\begin{align}
\SignalX{X}{z}
& = \sum_{n \in \Z} \signalX{x}{n} z^{-n} 
\\
& = \sum_{r \in \range{R}} \sum_{m \in \Z} \signalX{x}{m R + r} z^{-(m R + r)} \label{proof:polyphase-decomposition-step-1} \\
& = \sum_{r \in \range{R}} \left( \sum_{m \in \Z} \signalX{x}{m R + r} z^{-m R} \right) z^{-r}
\\
& = \sum_{r \in \range{R}} \left( \sum_{m \in \Z} \signalSupX{x}{\polyphase{r}{R}}{m} (z^R)^{-m} \right) z^{-r} 
\label{proof:polyphase-decomposition-step-2} \\
& = \sum_{r \in \range{R}} 
\SignalSupX{X}{\polyphase{r}{R}}{z^R} z^{-r},
\end{align}
where \Cref{proof:polyphase-decomposition-step-1} makes a change of variables $n = m R + r$, and \Cref{proof:polyphase-decomposition-step-2} is the definition of polyphase components $\signalSupX{x}{\polyphase{r}{R}}{m} = \signalX{x}{m R + r}$.
\end{proof}

For simplicity, we stack $R$ consecutive polyphase components into a \textbf{polyphase matrix} as
\begin{equation}
\label{eq:polyphase-matrices}
\SignalSupX{X}{\range{R}}{z} = 
\begin{bmatrix}
\SignalSupX{X}{\polyphase{0}{R}}{z} \\ \vdots \\
\SignalSupX{X}{\polyphase{R-1}{R}}{z}
\end{bmatrix}, \mathrm{or}~
\SignalSupX{\widetilde{X}}{\range{R}}{z} = 
\left[
\SignalSupX{X}{\polyphase{-0}{R}}{z}; \cdots; \SignalSupX{X}{\polyphase{-(R-1)}{R}}{z} \right].
\end{equation}

The following proposition extends the Parseval's theorem in \Cref{thm:parseval-theorem} and shows that the sequence norm $\norm{\signal {x}}$ can also be computed in terms of the polyphase matrix $\SignalSupX{X}{\range{R}}{z}$.

\begin{proposition}[Parseval's theorem for polyphase matrix]
Given a sequence $\signal{x}$ and its polyphase matrix $\signalSupX{X}{\range{R}}{z}$, the sequence norm $\norm{\signal{x}}$ can be computed by $\SignalSupX{X}{\range{R}}{e^{\im \omega}}$ in the spectral domain as
\label{thm:polyphase-parseval}
\begin{equation}
\label{eq:polyphase-parseval}
\Norm{\signal{x}}^2 =
\frac{1}{2\pi} \int_{-\pi}^{\pi}
\Norm{\SignalSupX{X}{\range{R}}{e^{\im \omega}}}^2 d\omega.
\end{equation}
\end{proposition}

\begin{proof}[Proof of \Cref{thm:polyphase-parseval}.]
The proof follows the standard Parseval's theorem in \Cref{thm:parseval-theorem}.
\begin{align}
\Norm{\signal{x}}^2
& = \sum_{n \in \Z} \Norm{\signalX{x}{n}}^2 
\label{proof:polyphase-parseval-step-1} \\
& = \sum_{r \in [R]} \sum_{m \in \Z} 
\Norm{\signalX{x}{m R + r}}^2 
\label{proof:polyphase-parseval-step-2} \\
& = \sum_{r \in [R]} \sum_{m \in \Z} \Norm{\signalSupX{x}{\polyphase{r}{R}}{m}}^2 
\label{proof:polyphase-parseval-step-3} \\ 
& = \frac{1}{2\pi} \sum_{r \in [R]} \int_{-\pi}^{\pi} \Norm{\SignalSupX{X}{\polyphase{r}{R}}{e^{\im\omega}}}^2 d\omega
\label{proof:polyphase-parseval-step-4} \\
& = \frac{1}{2\pi} \int_{-\pi}^{\pi}
\Norm{\SignalSupX{X}{\range{R}}{e^{\im \omega}}}^2 d\omega,
\end{align}
where \Cref{proof:polyphase-parseval-step-1} follows from the definition of sequence norm (\Cref{def:sequence-inner-prodct-and-norm}),
\Cref{proof:polyphase-parseval-step-2} makes a change of variables $n = m R + r$, \Cref{proof:polyphase-parseval-step-3} is the definition of polyphase components, and \Cref{proof:polyphase-parseval-step-4} applies Parseval's theorem (\Cref{thm:parseval-theorem}) to $\signalSup{x}{\polyphase{r}{R}}$'s,
\end{proof}
\subsection{Unifying Various Convolutional Layers in the Spectral Domain}
\label{app-sub:framework-unification}

In \Cref{app-sec:paraunitary}, the convolution theorem states that a standard convolutional layer is a matrix-vector product $\SignalX{Y}{z} = \FilterX{H}{z} \SignalX{X}{z}$ in the spectral domain, and the layer is orthogonal if and only if $\FilterX{H}{z}$ is paraunitary (\Cref{thm:paraunitary-conv1d}). However, {\em the canonical convolution theorem does not hold for variants of convolutions, thus enforcing a paraunitary $\FilterX{H}{z}$ may not lead to orthogonal convolution}. In this subsection, we address this limitation by showing that various convolutions can be uniformly written as $\FilterX{\underline{Y}}{z} = \FilterX{\underline{H}}{z} \FilterX{\underline{X}}{z}$, where $\FilterX{\underline{Y}}{z}$, $\FilterX{\underline{H}}{z}$, $\FilterX{\underline{X}}{z}$ are some spectral representations of $\signal{y}$, $\filter{h}$, $\signal{x}$. Subsequently, we prove that any of these layers is orthogonal if and only if its $\FilterX{\underline{H}}{z}$ is paraunitary.

\vspace{0.3em}
\textbf{(1) Strided convolutional layers} are widely used in neural networks to adjust the feature resolution: a strided convolution layer decreases the resolution by down-sampling after a standard convolution, while a transposed convolution increases the resolution by up-sampling before a standard convolution.

Formally, a strided convolutional layer with stride $R$ (abbrev.\ as $\downarrow\!R$-strided convolution) computes its output following
\begin{equation}
\label{eq:strided-conv1d-matrix}
\signalX{y}{i} = \sum_{n \in \Z} \filterX{h}{n} \signalX{x}{R i - n}.
\end{equation}
In contrast, a transposed strided convolutional layer with stride $R$ (abbrev.\ as $\uparrow\!R$-strided convolution) computes its output according to
\begin{equation}
\label{eq:transposed-conv1d-matrix}
\signalX{y}{i} = \sum_{n \in \Z} \filterX{h}{n} \signalSupX{x}{\uparrow R}{i - n}.
\end{equation}

\begin{proposition}[Orthogonality of strided convolutional layers]
\label{prop:strided-conv1d}
For a $\downarrow\!R$-strided convolution,
the spatial convolution in \Cref{eq:strided-conv1d-matrix}
leads to the following spectral representation:
\begin{equation}
\label{eq:strided-conv1d-matrix-z}
\FilterX{Y}{z} = \FilterSupX{\widetilde{H}}{\range{R}}{z} \SignalSupX{X}{\range{R}}{z}
\end{equation}
And for an $\uparrow\!R$-strided convolution, the spatial convolution
is represented in spectral domain as:
\begin{equation}
\label{eq:transposed-conv1d-matrix-z}
\SignalSupX{Y}{\range{R}}{z} = \FilterSupX{H}{\range{R}}{z} \SignalX{X}{z}
\end{equation}
Furthermore, a $\downarrow\!R$-strided convolution is orthogonal if and only if $\FilterSupX{\widetilde{H}}{\range{R}}{z}$ is paraunitary, and an $\uparrow\!R$-strided convolution is orthogonal if and only if $\filterSupX{H}{\range{R}}{z}$ is paraunitary.
\end{proposition}

\begin{proof}[Proof of \Cref{prop:strided-conv1d}]
\textbf{(1a) $\downarrow\!R$-strided convolutions.} 
We first prove the spectral representation of $\downarrow\!R$-strided convolution in \Cref{eq:strided-conv1d-matrix-z}.
\begin{align}
\SignalX{Y}{z} 
& = \sum_{i \in \Z} \signalX{y}{i} z^{-i} 
  = \sum_{i \in \Z} \left( \sum_{n \in \Z} \filterX{h}{n} \signalX{x}{R i - n} \right) z^{-i}
\label{proof-1:strided-conv1d-step-1} \\
& = \sum_{i \in \Z} \left( \sum_{r \in \range{R}} \sum_{m \in \Z} \filterX{h}{m R - r} \signalX{x}{(i - m)R + r} \right) z^{-i} 
\label{proof-1:strided-conv1d-step-2} \\
& = \sum_{r \in \range{R}} \left( \sum_{m \in \Z} \filterX{h}{m R - r} z^{-m} \left(\sum_{i} \signalX{x}{(i - m) R + r} z^{-(i - m)} \right) \right) 
\\
& = \sum_{r \in \range{R}} \left( \sum_{m \in \Z} \filterX{h}{m R - r} z^{-m} \right) \left( \sum_{i^{\prime} \in \Z} \signalX{x}{i^{\prime} R + r} z^{-i^{\prime}} \right) \label{proof-1:strided-conv1d-step-3} \\
& = \sum_{r \in \range{R}} \FilterSupX{H}{\polyphase{-r}{R}}{z} \SignalSupX{X}{\polyphase{r}{R}}{z},
\label{proof-1:strided-conv1d-step-4}
\end{align}
where \Cref{proof-1:strided-conv1d-step-1} follows from the definitions of the $\downarrow \!R$-strided convolution (\Cref{eq:strided-conv1d-matrix}) and the \ZTLong (\Cref{eq:z-transform}), \Cref{proof-1:strided-conv1d-step-2} makes a change of variables $n = m R - r$, \Cref{proof-1:strided-conv1d-step-3} further changes $i^{\prime} = i - m$, and \Cref{proof-1:strided-conv1d-step-4} is due to the definition of polyphase components (\Cref{eq:polyphase-component}).
Now We rewrite the last equation concisely as
\begin{equation}
\SignalX{Y}{z}
= \underbrace{\begin{bmatrix}
\FilterSupX{H}{\polyphase{-0}{R}}{z}; 
\cdots; 
\FilterSupX{H}{-(R-1)}{z} \end{bmatrix}}_{\textstyle \FilterSupX{\widetilde{H}}{\range{R}}{z} \mathstrut} ~ 
\underbrace{\begin{bmatrix}
\SignalSupX{X}{\polyphase{0}{R}}{z} \\ 
\vdots \\
\SignalSupX{X}{\polyphase{R-1}{R}}{z} 
\end{bmatrix}}_{\textstyle \SignalSupX{X}{\range{R}}{z} \mathstrut},
\end{equation}
which is the spectral representation of $\downarrow\!R$-strided convolutions in \Cref{eq:strided-conv1d-matrix-z}.

Now we prove the orthogonality condition for $\downarrow\!R$-strided convolutions.
\begin{align}
\Norm{\signal{y}}^2
& = \frac{1}{2\pi} \int_{-\pi}^{\pi}
\Norm{\SignalX{Y}{e^{\im \omega}}}^2 d\omega
\label{proof-2:strided-conv1d-step-1} \\
& = \frac{1}{2\pi} \int_{-\pi}^{\pi}
\Norm{\FilterSupX{\widetilde{H}}{\range{R}}{e^{\im \omega}} \SignalSupX{X}{\range{R}}{e^{\im \omega}}}^2 d\omega 
\label{proof-2:strided-conv1d-step-2} \\
& = \frac{1}{2\pi} \int_{-\pi}^{\pi}
\Norm{\SignalSupX{X}{\range{R}}{e^{\im \omega}}}^2 d\omega
\label{proof-2:strided-conv1d-step-3} \\
& = \Norm{\signal{x}}^2,
\label{proof-2:strided-conv1d-step-4}
\end{align}
where \Cref{proof-2:strided-conv1d-step-1,proof-2:strided-conv1d-step-4} are due to Parseval's theorems (\Cref{thm:parseval-theorem} and \Cref{thm:polyphase-parseval}), \Cref{proof-2:strided-conv1d-step-2} follows from the spectral representation of the $\downarrow\!R$-strided convolution (\Cref{eq:strided-conv1d-matrix-z}), and \Cref{proof-2:strided-conv1d-step-3} utilizes that the transfer matrix is unitary at each frequency. The ``only if'' part can be proved by contradiction similar to \Cref{thm:paraunitary-conv1d}.

\textbf{(1b) $\uparrow\!R$-strided convolutions.} According to \Cref{prop:up-sampling-z}, the \ZTLong of $\signalSup{x}{\uparrow R}$ is $\SignalX{X}{z^R}$. Therefore, an application of the convolution theorem (\Cref{thm:conv-theorem}) on \Cref{eq:transposed-conv1d-matrix} leads us to
\begin{equation}
\SignalX{Y}{z} = \FilterX{H}{z} \SignalX{X}{z^{R}}
\end{equation}
Expanding $\SignalX{Y}{z}$ and $\FilterX{H}{z}$ using polyphase decomposition (\Cref{prop:polyphase-decomposition}), we have
\begin{align}
\sum_{r \in \range{R}} \SignalSupX{Y}{\polyphase{r}{R}}{z^R} z^{-r} & = 
\left( \sum_{r \in \range{R}} \FilterSupX{H}{\polyphase{r}{R}}{z^R} z^{-r} \right) \SignalX{X}{z^R} 
\\
\sum_{r \in \range{R}} \SignalSupX{Y}{\polyphase{r}{R}}{z^R} z^{-r} & = 
\sum_{r \in \range{R}} \left(\FilterSupX{H}{\polyphase{r}{R}}{z^R} \SignalX{X}{z^R} \right) z^{-r} 
\\
\SignalSupX{Y}{\polyphase{r}{R}}{z^R} & = \FilterSupX{H}{\polyphase{r}{R}}{z^R} \SignalX{X}{z^R}, ~\forall r \in \range{R} 
\label{proof-1:transposed-conv1d-step-1} \\
\SignalSupX{Y}{\polyphase{r}{R}}{z} &
= \FilterSupX{H}{\polyphase{r}{R}}{z} \SignalX{X}{z}, ~\forall r \in \range{R},
\label{proof-1:transposed-conv1d-step-2}
\end{align}
where \Cref{proof-1:transposed-conv1d-step-1} is due to the uniqueness of \ZTLong, and \Cref{proof-1:transposed-conv1d-step-2} changes the variables from $z^{R}$ to $z$.
Again,,we can rewrite the last equation in concisely as
\begin{equation}
\underbrace{\begin{bmatrix}
\SignalSupX{Y}{\polyphase{0}{R}}{z} \\ 
\SignalSupX{Y}{\polyphase{1}{R}}{z} \\ \vdots \\ 
\SignalSupX{Y}{\polyphase{R-1}{R}}{z}
\end{bmatrix}}_{\textstyle \SignalSupX{Y}{\range{R}}{z} \mathstrut}
= \underbrace{\begin{bmatrix}
\FilterSupX{H}{\polyphase{0}{R}}{z} \\ 
\FilterSupX{H}{\polyphase{1}{R}}{z} \\ \vdots \\ 
\FilterSupX{H}{\polyphase{R-1}{R}}{z}
\end{bmatrix}}_{\textstyle\FilterSupX{H}{\range{R}}{z} \mathstrut} \SignalX{X}{z},
\end{equation}
which is the spectral representation of $\uparrow\!R$-strided convolutions in \Cref{eq:transposed-conv1d-matrix-z}.

Lastly, we prove the orthogonality condition for $\uparrow\!R$-strided convolutions.
\begin{align}
\Norm{\signal{y}}^2
& = \frac{1}{2\pi} \int_{-\pi}^{\pi}
\Norm{\SignalSupX{Y}{\range{R}}{e^{\im \omega}}}^2 d \omega 
\label{proof-2:transposed-conv1d-step-1} \\
& = \frac{1}{2\pi} \int_{-\pi}^{\pi}
\Norm{\SignalSupX{H}{\range{R}}{e^{j\omega}} \SignalX{X}{e^{\im \omega}}}^2 d \omega 
\label{proof-2:transposed-conv1d-step-2} \\
& = \frac{1}{2\pi} \int_{-\pi}^{\pi}
\Norm{\SignalX{X}{e^{\im \omega}}}^2 d \omega 
\label{proof-2:transposed-conv1d-step-3} \\
& = \Norm{\signal{x}}^2,
\label{proof-2:transposed-conv1d-step-4}
\end{align}
where \Cref{proof-2:transposed-conv1d-step-1,proof-2:transposed-conv1d-step-4} are due to Parseval's theorems (\Cref{thm:parseval-theorem} and \Cref{thm:polyphase-parseval}), \Cref{proof-2:transposed-conv1d-step-2} follows from the spectral representation of the $\uparrow R$-strided convolution (\Cref{eq:transposed-conv1d-matrix-z}), and \Cref{proof-2:transposed-conv1d-step-3} uses the fact that the transfer matrix is unitary for each frequency. The ``only if'' part can be proved by contradiction similar to \Cref{thm:paraunitary-conv1d}.
\end{proof}

\vspace{0.3em}
\textbf{(2) Dilated convolutional layer} is proposed to increase the receptive field of a convolutional layer without extra parameters and computation. The layer up-samples its filter bank before convolution with the input. $R$-dilated convolutional layer)
computes its output with the following equation:
\begin{equation}
\label{eq:dilated-conv1d-matrix}
\signalX{y}{i} = \sum_{n \in \Z} \filterSupX{h}{\uparrow R}{n} \signalX{x}{i - n}
\end{equation}

\begin{proposition}[Orthogonality of dilated convolutional layer]
\label{prop:dilated-conv1d}
For an $R$-dilated convolution, the spatial convolution in \Cref{eq:dilated-conv1d-matrix} leads to a spectral representation as
\begin{equation}
\label{eq:dilated-conv1d-matrix-z}
\SignalX{Y}{z} = \SignalX{H}{z^R} \SignalX{X}{z},
\end{equation}
Furthermore, an $R$-dilated convolutional layer is orthogonal if and only if $\SignalX{H}{z^R}$ is paraunitary.
\end{proposition}

\begin{proof}[Proof of \Cref{prop:dilated-conv1d}]
According to \Cref{prop:up-sampling-z}, the \ZTLong of $\filterSup{h}{\uparrow R}$ is $\SignalX{H}{z^{R}}$. Therefore, the ``if'' part follows directly from the convolution theorem. The ``only if'' part can be proved by constructing a counterexample similar to \Cref{thm:paraunitary-conv1d}.
\end{proof}

Notice that $\SignalX{H}{z^{R}}$ is paraunitary if and only if $\SignalX{H}{e^{\im \omega}}$ is unitary for all frequency $\omega \in \R$, which is the same as $\Signal X{H}{z}$ being paraunitary. In other words, any filter bank that is orthogonal for a standard convolution is also orthogonal for a dilated convolution and vice versa.

\vspace{0.3em}
\textbf{(3) Group convolutional layer} is proposed to reduce the parameters and computations and used in many efficient architectures, including MobileNet, ShuffleNet.
The layer divides both input/output channels into multiple groups and restricts the connections within each group.

Formally, a group convolutional layer with $G$ groups (abbrev.\ as $G$-group convolutions) is parameterized by $G$ filter banks $\{\filterSup{h}{g}\}_{g \in \range{G}}$, each consists of $(T/G) \times (S/G)$ filters. The layer maps an $S$ channels input $\signal{x}$ to a $T$ channels output $\signal{y}$ according to
\begin{equation}
\label{eq:group-conv1d-matrix}
\signalX{y}{i} = \sum_{n \in \Z} \Blkdiag{\{\filterSupX{h}{g}{n}\}_{g \in \range{G}}} \signalX{x}{i - n},
\end{equation}
where $\blkdiag{\{\cdot\}}$ computes a block diagonal  matrix from a set of matrices.

\begin{proposition}[Orthogonality of group convolutional layer]
\label{prop:group-conv1d}
For a $G$-group convolution, the spatial convolution in \Cref{eq:group-conv1d-matrix} leads a spectral representation as
, their {\ZTlong}s satisfy
\begin{equation}
\label{eq:group-conv1d-matrix-z}
\SignalX{Y}{z} = \Blkdiag{\{\FilterSupX{H}{g}{z}\}_{g \in \range{G}}} \SignalX{X}{z},
\end{equation}
Furthermore, a $G$-group convolutional layer is orthogonal if and only if the block diagonal  matrix is paraunitary, i.e., each $\FilterSupX{h}{g}{z}$ is paraunitary.
\end{proposition}

\begin{proof}[Proof of \Cref{prop:group-conv1d}]
Due to the convolution theorem, it suffices to prove that the \ZTLong of a sequence of block diagonal matrices is also block diagonal in the spectral domain.
\begin{align}
\sum_{n \in \Z} 
\underbrace{\begin{bmatrix}
\filterSupX{h}{0}{n} & & \\
& \ddots & \\
& & \filterSupX{h}{G-1}{n} 
\end{bmatrix}}_{\textstyle \Blkdiag{\{\filterSupX{h}{g}{n}\}_{g \in \range{G}}} \mathstrut}
z^{-n} & = 
\begin{bmatrix}
\sum_{n \in \Z} \filterSupX{h}{0}{n} z^{-n} & & \\
& \ddots & \\
& & \sum_{n \in \Z} \filterSupX{h}{G-1}{n} z^{-n}
\end{bmatrix} \\
& = \underbrace{\begin{bmatrix}
\FilterSupX{H}{0}{z} & & \\
& \ddots & \\
& & \FilterSupX{H}{G-1}{z} 
\end{bmatrix}}_{\textstyle \Blkdiag{\{\FilterSupX{H}{g}{z}\}_{g \in \range{G}}} \mathstrut}.
\end{align}
As a result, we can write the orthogonality condition as 
\begin{align}
& \begin{bmatrix}
\FilterSupX{H}{0}{z} & & \\
& \ddots & \\
& & \FilterSupX{h}{G-1}{z} 
\end{bmatrix}^{\dagger}
\begin{bmatrix}
\FilterSupX{H}{0}{z} & & \\
& \ddots & \\
& & \FilterSupX{H}{G-1}{z} 
\end{bmatrix} 
\nonumber \\
= ~ & 
\begin{bmatrix}
\FilterSupX{H}{0}{z}^{\dagger} \FilterSupX{H}{0}{z} & & \\
& \ddots & \\
& & \FilterSupX{H}{G-1}{z}^{\dagger} \FilterSupX{H}{G-1}{z} 
\end{bmatrix} = \identity, ~\forall |z| = 1.
\end{align}
The equation implies $\FilterSupX{H}{g}{z}^{\dagger} \FilterSupX{H}{g}{z} = \identity, \forall |z| = 1, \forall g \in \range{G}$, i.e., each $\FilterSupX{H}{g}{z}$ is paraunitary.
\end{proof}
\subsection{Realizations of Paraunitary Systems}
\label{app-sub:framework-construction}

In this subsection, we first prove that all finite-length 1D-paraunitary systems can be represented in a factorized form. Next, we show how we can construct MD-paraunitary systems using 1D systems. Lastly, we study the relationship of existing approaches to paraunitary systems.

\subsubsection{Complete Factorization of 1D-Paraunitary Systems}
\label{app-sub2:framework-construction-factorization}

The classic theorem for spectral factorization of paraunitary systems is traditionally developed for causal systems~\citep{vaidyanathan1993multirate,kautsky1994matrix}.
Given a causal paraunitary system of length $L$ (i.e., polynomial in $z^{-1}$), there always exists a factorization such that 
\begin{equation}
\label{eq:paraunitary-factorization-causal}
\FilterX{H}{\myvector{z}} = \mathbf{Q} \FilterX{V}{z^{-1}; \matrixSup{U}{1}} \cdots \FilterX{V}{z^{-1}; \matrixSup{U}{L-1}},
\end{equation}
where $\mymatrix{Q}$ is an orthogonal matrix, $\matrixSup{U}{\ell}$ is a column-orthogonal matrix, and $\FilterX{V}{z; \mymatrix{U}}$ is defined as
\begin{equation}
\label{eq:v-block}
\FilterX{V}{z; \mymatrix{U}} = (\identity - \mymatrix{U} \mymatrix{U}^{\top}) + \mymatrix{U} \mymatrix{U}^{\top} z.
\end{equation}
In \Cref{thm:paraunitary-factorization}, we extends this theorem from causal systems to finite-length (but non-causal) ones.

\begin{theorem}[Complete factorization for 1D-paraunitary systems]
\label{thm:paraunitary-factorization}
Suppose that a paraunitary system $\FilterX{H}{z}$ is finite-length, i.e., it can be written as $\sum_{n} \filterX{h}{n} z^{-n}$ for some sequence $\{\filterX{h}{n}, n \in [-\underline{L}, \overline{L}]\}$, then it can be factorized in the following form:
\begin{equation}
\label{eq:paraunitary-factorization}
\FilterX{H}{\myvector{z}} = \FilterX{V}{z; \matrixSup{U}{-\underline{L}}} \cdots \FilterX{V}{z; \matrixSup{U}{-1}} \mathbf{Q} \FilterX{V}{z^{-1}; \matrixSup{U}{1}} \cdots \FilterX{V}{z^{-1}; \matrixSup{U}{\overline{L}}},
\end{equation}
where $\mymatrix{Q}$ is an orthogonal matrix, $\matrixSup{U}{\ell}$ is a column-orthogonal matrix, and $\FilterX{V}{z; \mymatrix{U}}$ is defined in \Cref{eq:v-block}. Consequently, the paraunitary system $\FilterX{H}{z}$ is parameterized by $\underline{L} + \overline{L} + 1$ (column-)orthogonal matrices $\mymatrix{Q}$ and $\matrixSup{U}{\ell}$'s. 
\end{theorem}

\begin{proof}[Proof for \Cref{thm:paraunitary-factorization}]
Given a non-causal paraunitary system $\FilterX{H}{z}$, we can always find a causal counterpart $\FilterX{\hat{H}}{z}$ such that $\FilterX{H}{z} = z^{\underline{L}} \FilterX{\hat{H}}{z}$ (This can be done by shifting the causal system backward by $\underline{L}$ steps, which is equivalent to multiplying $z^{\underline{L}}$ in the spectral domain). Since the causal system $\FilterX{\hat{H}}{z}$ admits a factorization in \Cref{eq:paraunitary-factorization}, we can write the non-causal system $\FilterX{H}{z}$ as
\begin{equation}
\label{proof:paraunitary-factorization-step-1}
\FilterX{H}{z} = z^{\underline{L}} \mymatrix{Q} \FilterX{V}{z^{-1}; \matrixSup{\hat{U}}{1}} \cdots \FilterX{V}{z^{-1}; \matrixSup{\hat{U}}{\underline{L} + \overline{L}}}.
\end{equation}
Therefore, it suffices to show that for an orthogonal matrix $\mymatrix{Q}$ and any column-orthogonal matrix $\mymatrix{\hat{U}}$, we can always find another column-orthogonal matrix $\mymatrix{U}$ such that
\begin{equation}
\label{proof:paraunitary-factorization-step-2}
z \mymatrix{Q} \FilterX{V}{z^{-1}; \mymatrix{\hat{U}}} = \FilterX{V}{z; \mymatrix{U}} \mymatrix{Q}.
\end{equation}
If the equation above is true, we can set $\matrixSup{U}{\ell} = \matrixSup{\hat{U}}{\ell - 1 - \underline{L}}$ for $\ell < 0$ and $\matrixSup{U}{\ell} = \matrixSup{\hat{U}}{\ell - \underline{L}}$ for $\ell > 0$, which will convert \Cref{proof:paraunitary-factorization-step-1} into \Cref{eq:paraunitary-factorization}.

Now we start to prove \Cref{proof:paraunitary-factorization-step-2}. Note that any column-orthogonal $\mymatrix{\hat{U}}$ has a complement $\mymatrix{\bar{U}}$ such that $[\mymatrix{\hat{U}}, \mymatrix{\bar{U}}]$ is orthogonal and $\identity = \mymatrix{\hat{U}} \mymatrix{\hat{U}}^{\top} + \mymatrix{\bar{U}} \mymatrix{\bar{U}}^{\top}$. We then rewrite \Cref{proof:paraunitary-factorization-step-2} as
\begin{align}
z \mymatrix{Q} \FilterX{V}{z^{-1}; \mymatrix{\hat{U}}} & = 
z \mymatrix{Q} (\identity - \mymatrix{\hat{U}} \mymatrix{\hat{U}}^{\top} + \mymatrix{\hat{U}} \mymatrix{\hat{U}}^{\top} z^{-1}) 
\\
& = \mymatrix{Q} (\identity - \mymatrix{\bar{U}} \mymatrix{\bar{U}}^{\top} + \mymatrix{\bar{U}} \mymatrix{\bar{U}}^{\top} z) 
\\
& = (\identity - \mymatrix{Q} \mymatrix{\bar{U}} \mymatrix{\bar{U}}^{\top} \mymatrix{Q}^{\top} + \mymatrix{Q} \mymatrix{\bar{U}} \mymatrix{\bar{U}}^{\top} \mymatrix{Q}^{\top} z) \mymatrix{Q} 
\\
& = (\identity - \mymatrix{U} \mymatrix{U}^{\top} + \mymatrix{U} \mymatrix{U}^{\top} z) \mymatrix{Q} 
\label{proof:paraunitary-factorization-step-3} \\
& = \FilterX{V}{z; \mymatrix{U}} \mymatrix{Q},
\end{align}
where in \Cref{proof:paraunitary-factorization-step-3} we set $\mymatrix{U} = \mymatrix{Q} \mymatrix{\bar{U}}$. 
This completes the proof.
\end{proof}

\subsubsection{Multi-dimensional (MD) Paraunitary Systems}
\label{app-sub2:framework-construction-md}

If the data are multi-dimensional (MD), we will need MD-convolutional layers in neural networks. Analogously, we can prove the equivalence between orthogonal MD-convolutions in the spatial domain and MD-paraunitary systems in the spectral domain, i.e.,
\begin{equation}
\FilterX{H}{\myvector{z}}^{\dagger} \FilterX{H}{\myvector{z}} = \identity, \myvector{z} = (z_1, \cdots, z_D), |z_d| = 1, \forall d \in \range{D},
\end{equation}
where $D$ is the data dimension. In this work, we adopt a parameterization based on separable systems.

\begin{definition}[Separable MD-paraunitary system]
\label{def:seperable-paraunitary-md}
A MD-paraunitary system $\FilterX{H}{\myvector{z}}$ is separable if there exists $D$ 1D-paraunitary systems $\FilterX{H_1}{z_1}, \cdots, \FilterX{H_D}{z_D}$ such that 
\begin{equation}
\label{eq:paraunitary-separable}
\FilterX{H}{\myvector{z}} = \FilterX{H}{z_1, \cdots, z_D}
\triangleq \FilterX{H_1}{z_1} \cdots \FilterX{H_D}{z_D}.
\end{equation}
\end{definition}
Therefore, we can construct an MD-paraunitary system with $D$ number of 1D-paraunitary systems, each of which is represented in \Cref{eq:paraunitary-factorization}.
Notice that {\em not} all MD-paraunitary systems are separable, thus the parameterization in \Cref{eq:paraunitary-separable} is {\em not} complete (see \Cref{sec:related} for a discussion). However, we can guarantee that our parameterization realizes all separable MD-paraunitary systems --- each separable paraunitary system admits a factorization in \Cref{eq:paraunitary-separable}, where each 1D-system admits a factorization in \Cref{eq:paraunitary-factorization}.

\subsubsection{Interpretations of Previous Approaches}
\label{app-sub2:framework-construction-interpretation}

In \Cref{thm:paraunitary-conv1d}, we have shown that a paraunitary transfer matrix is both necessary and sufficient for a convolution to be orthogonal.
Therefore, we can interpret existing approaches for orthogonal convolutions, including {\em \SVCMLONG (\SVCM)}~\citep{sedghi2018singular}, {\em \BCOPLONG (\BCOP)}~\citep{li2019preventing}, {\em \CayleyConvLONG (\CayleyConv)}~\citep{trockman2021orthogonalizing}, as implicit constructions of paraunitary systems.

\vspace{0.3em}
{\bf(1) \SVCMLong (\SVCM)~
\citep{sedghi2018singular}} clips all singular values of $\FilterX{H}{e^{\im \omega}}$ to ones for all frequencies $\omega$ after gradient update. Since the clipping step can arbitrarily enlarge the filter length, \SVCM subsequently masks out the coefficients outside the filter length. However, the masking step breaks the orthogonality, as we have seen in the experiments (\Cref{sub:exact-orthogonality}).

\vspace{0.3em}
{\bf(2) \BCOPLong (\BCOP)~\citep{li2019preventing}} tries to generalize the spectral factorization of 1D-paraunitary systems in \Cref{eq:paraunitary-factorization} to the 2D paraunitary systems.
\begin{equation}
\label{eq:paraunitary-bcop}
\FilterX{H}{z_1, z_2} = z_1^{\frac{K - 1}{2}} z_2^{\frac{K - 1}{2}}
\mymatrix{Q} \FilterX{V}{z_1, z_2; \mymatrix{U}_1^{(1)}, \mymatrix{U}_2^{(1)}} \cdots \FilterX{V}{z_1, z_2; \mymatrix{U}_1^{(K-1)}, \mymatrix{U}_2^{(K-1)}},
\end{equation}
where $\FilterX{V}{z_1, z_2; \mymatrix{U}_1^{(\ell)}, \mymatrix{U}_2^{(\ell)}} = \FilterX{V}{z_1; \mymatrix{U}_1^{(\ell)}} \FilterX{V}{z_2; \mymatrix{U}_2^{(\ell)}}$. In other words, this approach makes each $V$-block, instead of the whole paraunitary system, separable. It is known that the factorization in \Cref{eq:paraunitary-bcop} is not complete for 2D-paraunitary system~\citep{lin1996theory}).

\vspace{0.3em}
{\bf(3) \CayleyConvLong (\CayleyConv)}~\citep{trockman2021orthogonalizing} to generalizes the Cayley transform for orthogonal matrices (\Cref{eq:orthogonal-cayley}) to 2D-paraunitary systems $\FilterX{H}{z_1,z_2}$:
\begin{equation}
\label{eq:paraunitary-cayley-conv}
\FilterX{H}{z_1, z_2} = \left(\identity - \FilterX{A}{z_1, z_2} \right) \left(\identity + \FilterX{A}{z_1, z_2} \right)^{-1},
\end{equation}
where $\FilterX{A}{z_1, z_2}$ is a skew-symmetric matrix for any $z_1, z_2$. Since a matrix with singular value $-1$ cannot be realized by Cayley transform, this parameterization is incomplete. Furthermore, the approach requires matrix inversion at each frequency, which is expensive in practice.

\vspace{0.3em}
{\bf(4) Orthogonal regularization~\cite{wang2019orthogonal,qi2020deep}} is developed to encourage orthogonality in convolutional layers.  
We will show that this regularization is equivalent to a unitary regularization of the paraunitary system with uniform weights on all frequencies:
\begin{equation}
\label{eq:paraunitary-regularization}
\sum_{i \in \Z} \Norm{\sum_{n \in \Z} \signalX{h}{n}
\signalX{h}{n - R i}^{\top} - \signalX{\bm{\delta}}{i}}^2
= \frac{1}{2\pi} \int_{-\pi}^{\pi}
\Norm{\adjoint{\FilterSupX{H}{\range{R}}{e^{\im \omega}}} \FilterSupX{H}{\range{R}}{e^{\im \omega}} - \identity}^2 d \omega,
\end{equation}
where $\filterX{\bm{\delta}}{0} = \identity$ is an identity matrix, $\filterX{\bm{\delta}}{n} = \zero$ is a zero matrix for $n \neq 0$. 
We prove the equivalence more generally in \Cref{prop:parseval-regularization}.
However, this approach cannot enforce exact orthogonality and in practice requires hyperparameter search for a proper regularizer coefficient.

\begin{proposition}[Parseval's theorem for ridge regularization]
\label{prop:parseval-regularization}
Given a sequence of matrices $\filter{h} = \{\filterX{h}{n}, n \in \Z\}$, the following four expressions are equivalent:
\begin{subequations}
\begin{gather}
\frac{1}{2\pi} \int_{-\pi}^{\pi} \Norm{ \adjoint{\FilterSupX{H}{\range{R}}{e^{\im \omega}}} \FilterSupX{H}{\range{R}}{e^{\im \omega}}  - \identity}^2 d \omega,
\label{eq:parseval-regularization-1} \\
\sum_{i \in \Z} \Norm{\sum_{n \in \Z} \signalX{h}{n}^{\top} \signalX{h}{n - R i} - \signalX{\bm{\delta}}{i}}^2,
\label{eq:parseval-regularization-2} \\
\frac{1}{2\pi} \int_{-\pi}^{\pi} \Norm{ \FilterSupX{H}{\range{R}}{e^{\im \omega}} \adjoint{\FilterSupX{H}{\range{R}}{e^{\im \omega}}} - \identity}^2 d \omega,
\label{eq:parseval-regularization-3} \\
\sum_{i \in \Z} \Norm{\sum_{n \in \Z} \signalX{h}{n} \signalX{h}{n - R i}^{\top} - \signalX{\bm{\delta}}{i}}^2,
\label{eq:parseval-regularization-4}
\end{gather}
\end{subequations}
where $\norm{\cdot}$ denotes the Frobenius norm of a matrix.
\end{proposition}

\begin{proof}[Proof of \Cref{prop:parseval-regularization}]
We first prove the equivalence between \Cref{eq:parseval-regularization-1,eq:parseval-regularization-3}.
\begin{align}
& \Norm{\adjoint{\FilterSupX{H}{\range{R}}{e^{\im \omega}}} \FilterSupX{H}{\range{R}}{e^{\im \omega}} - \identity}^2
\nonumber \\
= ~ & \Tr{\left(\adjoint{\FilterSupX{H}{\range{R}}{e^{\im \omega}}} \FilterSupX{H}{\range{R}}{e^{\im \omega}} - \identity \right)^{\dagger} \left(\adjoint{\FilterSupX{H}{\range{R}}{e^{\im \omega}}} \FilterSupX{H}{\range{R}}{e^{\im \omega}} - \identity \right)}
\label{proof-1:parseval-regularization-step-1} \\
= ~ & \Tr{\adjoint{\FilterSupX{H}{\range{R}}{e^{\im \omega}}} \FilterSupX{H}{\range{R}}{e^{\im \omega}}\adjoint{\FilterSupX{H}{\range{R}}{e^{\im \omega}}} \FilterSupX{H}{\range{R}}{e^{\im \omega}}} - 2 \Tr{\adjoint{\FilterSupX{H}{\range{R}}{e^{\im \omega}}} \FilterSupX{H}{\range{R}}{e^{\im \omega}}} + \identity
\label{proof-1:parseval-regularization-step-2} \\
= ~ & \Tr{ \FilterSupX{H}{\range{R}}{e^{\im \omega}}\adjoint{\FilterSupX{H}{\range{R}}{e^{\im \omega}}} \FilterSupX{H}{\range{R}}{e^{\im \omega}}\adjoint{\FilterSupX{H}{\range{R}}{e^{\im \omega}}}} - 2 \Tr{ \FilterSupX{H}{\range{R}}{e^{\im \omega}}\adjoint{\FilterSupX{H}{\range{R}}{e^{\im \omega}}}} + \identity
\label{proof-1:parseval-regularization-step-3} \\
= ~ & \Tr{\left( \FilterSupX{H}{\range{R}}{e^{\im \omega}}\adjoint{\FilterSupX{H}{\range{R}}{e^{\im \omega}}} - \identity \right)^{\dagger} \left(\FilterSupX{H}{\range{R}}{e^{\im \omega}}\adjoint{\FilterSupX{H}{\range{R}}{e^{\im \omega}}}  - \identity \right)}
\label{proof-1:parseval-regularization-step-4} \\
= ~ & \Norm{\FilterSupX{H}{\range{R}}{e^{\im \omega}}\adjoint{\FilterSupX{H}{\range{R}}{e^{\im \omega}}} - \identity}^2,
\label{proof-1:parseval-regularization-step-5}
\end{align}
where \Cref{proof-1:parseval-regularization-step-1,proof-1:parseval-regularization-step-5} make use of $\norm{\mymatrix{A}}^2 = \tr{\mymatrix{A}^{\dagger} \mymatrix{A}}$, \Cref{proof-1:parseval-regularization-step-2,proof-1:parseval-regularization-step-4} are due to the linearity of $\tr{\cdot}$, and \Cref{proof-1:parseval-regularization-step-3} utilizes $\tr{\mymatrix{A}\mymatrix{B}} = \tr{\mymatrix{B} \mymatrix{A}}$.

Next, we prove the equivalence between \Cref{eq:parseval-regularization-1,eq:parseval-regularization-2}.
\begin{align}
& \adjoint{\FilterSupX{H}{\range{R}}{e^{\im \omega}}} \FilterSupX{H}{\range{R}}{e^{\im \omega}} - \identity =
\sum_{r \in \range{R}} \adjoint{\FilterSupX{H}{\polyphase{r}{R}}{e^{\im \omega}}} \FilterSupX{H}{\polyphase{r}{R}}{e^{\im \omega}} - \identity
\label{proof-2:parseval-regularization-step-1} \\
= ~ & \sum_{r \in \range{R}} \sum_{i \in \Z} \left(\sum_{m \in \Z} \signalSupX{h}{\polyphase{r}{R}}{m}^{\top}  \signalSupX{h}{\polyphase{r}{R}}{m - i} - \signalX{\delta}{i} \right) e^{-\im \omega i} 
\label{proof-2:parseval-regularization-step-2} \\
= ~ & \sum_{r \in \range{R}} \sum_{i \in \Z} \left(\sum_{m \in \Z} \signalX{h}{R m + r}^{\top} \signalX{h}{R(m - i) + r} - \signalX{\delta}{i} \right) e^{-\im \omega i} 
\label{proof-2:parseval-regularization-step-3} \\
= ~ & \sum_{i \in \Z} \left( \sum_{r \in \range{R}} \sum_{m \in \Z} \signalX{h}{R m + r}^{\top} \signalX{h}{R m + r - R i}  - \signalX{\delta}{i} \right) e^{-\im \omega i}
\label{proof-2:parseval-regularization-step-4} \\
= ~ & \sum_{i \in \Z} \left( \sum_{n \in \Z} \signalX{h}{n}^{\top} \signalX{h}{n - R i} - \signalX{\delta}{i} \right) e^{-\im \omega i},
\label{proof-2:parseval-regularization-step-5}
\end{align}
where \Cref{proof-2:parseval-regularization-step-1} follows from the definition of polyphase matrix in \Cref{eq:polyphase-matrices}. \Cref{proof-2:parseval-regularization-step-2} uses a number of properties of Fourier transform: a Hermitian in the spectral domain is a transposed reflection in the spatial domain, a frequency-wise multiplication in the spectral domain is a convolution in the spatial domain, and an identical mapping in the spectral domain is an impulse sequence in the spatial domain. \Cref{proof-2:parseval-regularization-step-3} follows from the definition of polyphase components in \Cref{eq:polyphase-component}, and \Cref{proof-2:parseval-regularization-step-5} makes a change of variables $n = R m + r$.  In summary, we show that the LHS (denoted $\FilterX{D}{e^{\im \omega}}$) is a Fourier transform of the RHS (denoted as $\filterX{d}{i}$):  
\begin{equation}
\underbrace{\adjoint{\FilterSupX{H}{\range{R}}{e^{\im \omega}}} \FilterSupX{H}{\range{R}}{e^{\im \omega}} - \identity}_{\textstyle  \FilterX{D}{e^{\im \omega}} \mathstrut} =
\sum_{i \in \Z} \underbrace{\left(\sum_{n \in \Z} \signalX{h}{n}^{\top} \signalX{h}{n - R i} - \signalX{\delta}{n}\right)}_{\textstyle \filterX{d}{i} \mathstrut} e^{-\im \omega i}.
\end{equation}
Applying Parseval's theorem (\Cref{thm:parseval-theorem}) to the sequence $\filter{d} = \{\filterX{d}{i}, i \in \Z\}$, we have
\begin{equation}
\frac{1}{2\pi} \int_{-\pi}^{\pi} 
\Norm{\adjoint{\FilterSupX{H}{\range{R}}{e^{\im \omega}}} \FilterSupX{H}{\range{R}}{e^{\im \omega}}  - \identity}^2 d\omega
= \sum_{i \in \Z} \Norm{\sum_{n \in \Z} \signalX{h}{n}^{\top}
\signalX{h}{n - R i} - \signalX{\bm{\delta}}{i}}^2,
\end{equation}
which proves the equivalence between \Cref{eq:parseval-regularization-1,eq:parseval-regularization-2}. With almost identical arguments, we can prove the equivalence between \Cref{eq:parseval-regularization-3,eq:parseval-regularization-4}, that is
\begin{equation}
\frac{1}{2\pi} \int_{-\pi}^{\pi}
\Norm{\FilterSupX{H}{\range{R}}{e^{\im \omega}} 
\adjoint{\FilterSupX{H}{\range{R}}{e^{\im \omega}}} - \identity}^2 d\omega
= \sum_{i \in \Z} \Norm{\sum_{n \in \Z} \signalX{h}{n}
\signalX{h}{n - R i}^{\top} - \signalX{\bm{\delta}}{i}}^2,
\end{equation}
which completes the proof. 
\end{proof}
\subsection{Constrained Optimization over Orthogonal Matrices}
\label{app-sub:framework-parameterization}

In \Cref{thm:paraunitary-factorization}, we have shown how orthogonal matrices completely represent a paraunitary system. 
Our remaining goal, therefore, is to parameterize orthogonal matrices using unconstrained parameters.
In this part, we analyze three common approaches, {\em Björck orthogonalization}, {\em Cayley transform}, and {\em Lie exponential map}, and show that only Lie exponential map provides an exact and complete characterization of all orthogonal matrices.

\vspace{0.3em}
{\bf Björck orthogonalization~\citep{anil2019sorting,li2019preventing}.}
The algorithm was first introduced by Björck~\citep{bjorck1971iterative}. Given an initial matrix $\mymatrix{U}_{0}$, this algorithm finds the closest orthogonal matrix via an {\em iterative} process:
\begin{equation}
\label{eq:orthogonal-bjorck}
\mymatrix{U}_{k + 1} = \mymatrix{U}_{k} \left(\identity + \frac{1}{2} \mymatrix{P}_{k} + \cdots + (-1)^p \begin{pmatrix} -\frac{1}{2} \\ p \end{pmatrix} \mymatrix{P}^p_k \right), \forall k \in \range{K},
\end{equation}
where $K$ is the iterative steps, $\mymatrix{P}_k = \identity - \mymatrix{U}_k^{\top} \mymatrix{U}_k$, and $p$ controls the trade-off between efficiency and accuracy at each step.
When the algorithm is used for parameterization, it maps an unconstrained matrix $\mymatrix{U}_0$ to an approximately orthogonal matrix $\mymatrix{U}_{K}$ in $K$ steps. Although Björck parameterization is complete (since any orthogonal matrix $\mymatrix{Q}$ can be represented by $\mymatrix{U}_{0} = \mymatrix{Q}$), it is not exact due to the iterative approximation.

\vspace{0.3em}
{\bf Cayley transform~\citep{helfrich2018orthogonal,maduranga2019complex}.}
The transform provides a {\em bijective} parameterization of orthogonal matrices {\em without $-1$ eigenvalue} with skew-symmetric matrices (i.e., $\mymatrix{A}^{\top} = -\mymatrix{A}$)
\begin{equation}
\label{eq:orthogonal-cayley}
\mymatrix{U} = (\identity - \mymatrix{A}) (\identity + \mymatrix{A})^{-1},
\end{equation}
where the skew-symmetric matrix $\mymatrix{A}$ is represented by its upper-triangle entries. Since the orthogonal matrices with $-1$ eigenvalue are out of consideration, the Cayley parameterization is incomplete. Furthermore, the matrix inversion becomes unstable when there exists an eigenvalue close to $-1$.

\vspace{0.3em}
{\bf Lie Exponential Map~\citep{lezcano2019cheap,lezcano2019trivializations}.}
The Lie exponential map provides a {\em surjective} parameterization of all orthogonal matrices with skew-symmetry matrices (i.e., $\mymatrix{A}^{\top} = -\mymatrix{A}$).
\begin{equation}
\label{eq:orthogonal-matrix-exponential-appendix}
\mymatrix{U} = \exp(\mymatrix{A}) \triangleq \sum_{k = 0}^{\infty} \frac{\mymatrix{A}^{k}}{k!} = \identity + \mymatrix{A} + \frac{1}{2} \mymatrix{A}^2 +  \cdots, 
\end{equation}
where the infinite sum can be computed exactly up to machine-precision~\citep{higham2009scaling}.
\citet{lezcano2019cheap} further addresses the problem of gradient computation, thus makes Lie exponential map a complete and exact parameterization of orthogonal matrices.  

In principle, we can use any of these approach to parameterize the orthogonal matrices. In our implementation, we choose Lie exponential map due to its completeness and exactness.
\section{Learning Deep Orthogonal Networks with Lipschitz Bounds}
\label{app-sec:deep-networks}

In this section, we provide the proof for \Cref{prop:lipschitz-resnet}, which exhibits two approaches to construct Lipschitz residual blocks. Furthermore, we will prove in \Cref{prop:paraunitary-initialization} when a paraunitary system (represented in \Cref{thm:paraunitary-factorization}) reduces to an orthogonal matrix. The reduction allows us to apply the initialization methods for orthogonal matrices to paraunitary systems.

In \Cref{prop:lipschitz-resnet}, we prove the Lipschitzness of two types of residual blocks, one based on additive skip-connection and another based on concatenative one (See \Cref{fig:residual-blocks} for illustration).

\begin{proof}[Proof for \Cref{prop:lipschitz-resnet}]
To begin with, we prove the Lipschitzness for the additive residual block $f$ in \Cref{eq:resnet-block}.
Let $\signal{x}$, $\signal{x}^{\prime}$ be two inputs to $f$ and $f(\signal{x})$, $f(\signal{x}^{\prime})$ be their respective outputs,  we have
\begin{align}
\Norm{f(\signal{x}^{\prime}) - f(\signal{x})}
& = \Norm{\left(\alpha f^1(\signal{x}^{\prime}) + (1 - \alpha) f^2(\signal{x}^{\prime}) \right) - \left(\alpha f^1(\signal{x}) + (1 - \alpha) f^2(\signal{x}) \right)} 
\\
& = \Norm{\alpha \left(f^1(\signal{x}^{\prime}) - f^1(\signal{x}) \right) + (1 - \alpha) \left(f^2(\signal{x}^{\prime}) - f^2(\signal{x}) \right)}
\\
& \leq \alpha \Norm{f^1(\signal{x}^{\prime}) - f^1(\signal{x})} + (1 - \alpha) \Norm{ f^2(\signal{x}^{\prime}) - f^2(\signal{x})}
\label{proof:lipschitz-resnet-step-1} 
\\
& \leq \alpha L \Norm{\signal{x}^{\prime} - \signal{x}} + (1 - \alpha) L \Norm{ \signal{x}^{\prime} - \signal{x}}
\label{proof:lipschitz-resnet-step-2} 
\\
& = L \norm{\signal{x}^{\prime} - \signal{x}},
\end{align}
where \Cref{proof:lipschitz-resnet-step-1} makes uses of the triangle inequality, and \Cref{proof:lipschitz-resnet-step-2} is due to the $L$-Lipschitzness of both $f^1$, $f^2$. Therefore, we have shown that $\norm{f(\signal{x}^{\prime}) - f(\signal{x})} \leq L \norm{\signal{x}^{\prime} - \signal{x}}$.

Similarly, we can prove the Lipschitzness for the concatenative residual block $g$ in \Cref{eq:shufflenet-block}.
Let $\signal{x}$, $\signal{x}^{\prime}$ be two inputs to $g$ and $g(\signal{x})$, $g(\signal{x}^{\prime})$ be their respective outputs, we have
\begin{align}
\Norm{g(\signal{x}^{\prime}) - g(\signal{x})}^2
& = \Norm{\mymatrix{P} \left([g^1({\signalSup{x}{1}}^{\prime}); g^2({\signalSup{x}{2}}^{\prime})] - [g^1(\signalSup{x}{1}); g^2(\signalSup{x}{2})] \right)}
\\
& = \Norm{[g^1(\signal{x}^{\prime}); g^2(\signal{x}^{\prime})] - [g^1(\signal{x}); g^2(\signal{x})]}^2
\label{proof:lipschitz-shufflenet-step-1} \\
& = \Norm{g^1({\signalSup{x}{1}}^{\prime}) - g^1(\signal{x})}^2 + \Norm{ g^2({\signalSup{x}{2}}^{\prime}) - g^2(\signalSup{x}{2})}^2
\\
& \leq L^2 \Norm{{\signalSup{x}{1}}^{\prime} - \signalSup{x}{1}}^2 + L^2 \Norm {{\signalSup{x}{2}}^{\prime} - \signalSup{x}{2}}^2 
\label{proof:lipschitz-shufflenet-step-2} \\
& = L^2 \Norm{[{\signalSup{x}{1}}^{\prime}; {\signalSup{x}{2}}^{\prime}] - [{\signalSup{x}{1}}; {\signalSup{x}{2}}]}^2 
\\
& = L^2 \Norm{\signal{x}^{\prime} - \signal{x}}^2,
\end{align}
where \Cref{proof:lipschitz-shufflenet-step-1} utilizes $\norm{\mymatrix{P} \signal{x}} = \norm{\signal{x}}, \forall \signal{x}$, 
and \Cref{proof:lipschitz-shufflenet-step-2} is due to the $L$-Lipschitzness of $g^1$, $g^2$. The equations above implies that $\norm{g(\signal{x}^{\prime}) - g(\signal{x})} \leq L \norm{\signal{x}^{\prime} - \signal{x}}$.
\end{proof}

\vspace{0.3em}
\begin{proposition}[Initialization of a paraunitary matrix]
\label{prop:paraunitary-initialization}
Suppose a paraunitary system $\FilterX{H}{z}$ takes the complete factorization in \Cref{eq:paraunitary-factorization}, and if we further assume $\underline{L} = \overline{L}$ with $\matrixSup{U}{-\ell} = \mymatrix{Q} \matrixSup{U}{\ell}$ for all $\ell$, then the paraunitary matrix $\FilterX{H}{z}$ reduces to an orthogonal matrix $\mymatrix{Q}$,
\begin{equation}
\label{eq:paraunitary-initialization}
\FilterX{H}{z} = \FilterX{V}{z; \matrixSup{U}{-\underline{L}}} \cdots \FilterX{V}{z; \matrixSup{U}{-1}} \mymatrix{Q} \FilterX{V}{z^{-1}; \matrixSup{U}{1}} \cdots \FilterX{V}{z^{-1}; \matrixSup{U}{\overline{L}}} = \mymatrix{Q}.
\end{equation}
\end{proposition}

\begin{figure}[!htbp]
\centering
    \begin{subfigure}[b]{0.22\textwidth}
    \centering
        \includegraphics[trim={0.6cm 0.8cm 0.6cm 0.9cm },clip,width=\textwidth]{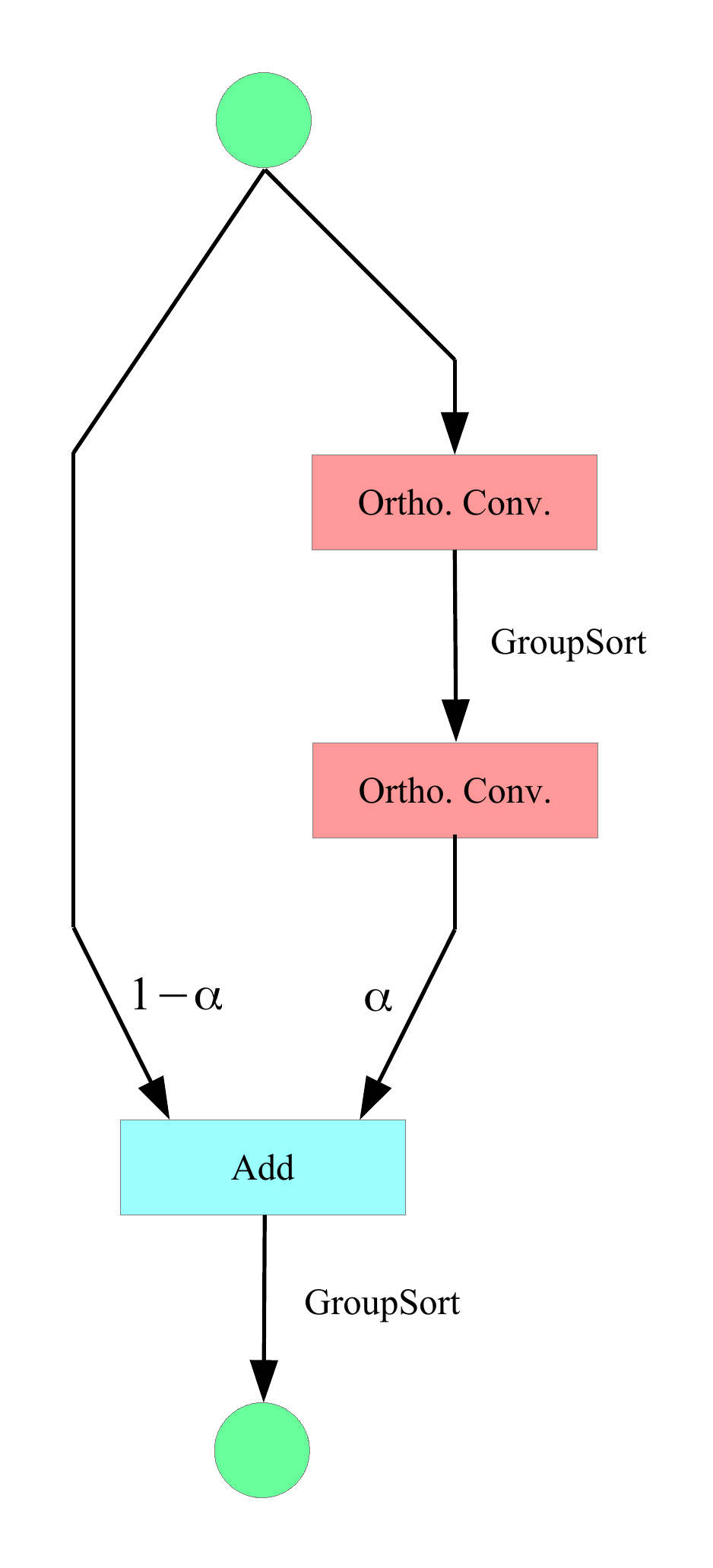}
        \caption{Basic additive block.}
        \label{subfig:resnet}
    \end{subfigure}
    \hfill
    \begin{subfigure}[b]{0.27\textwidth}
    \centering
        \includegraphics[trim={0.6cm 0.8cm 0.6cm 0.8cm },clip,width=\textwidth]{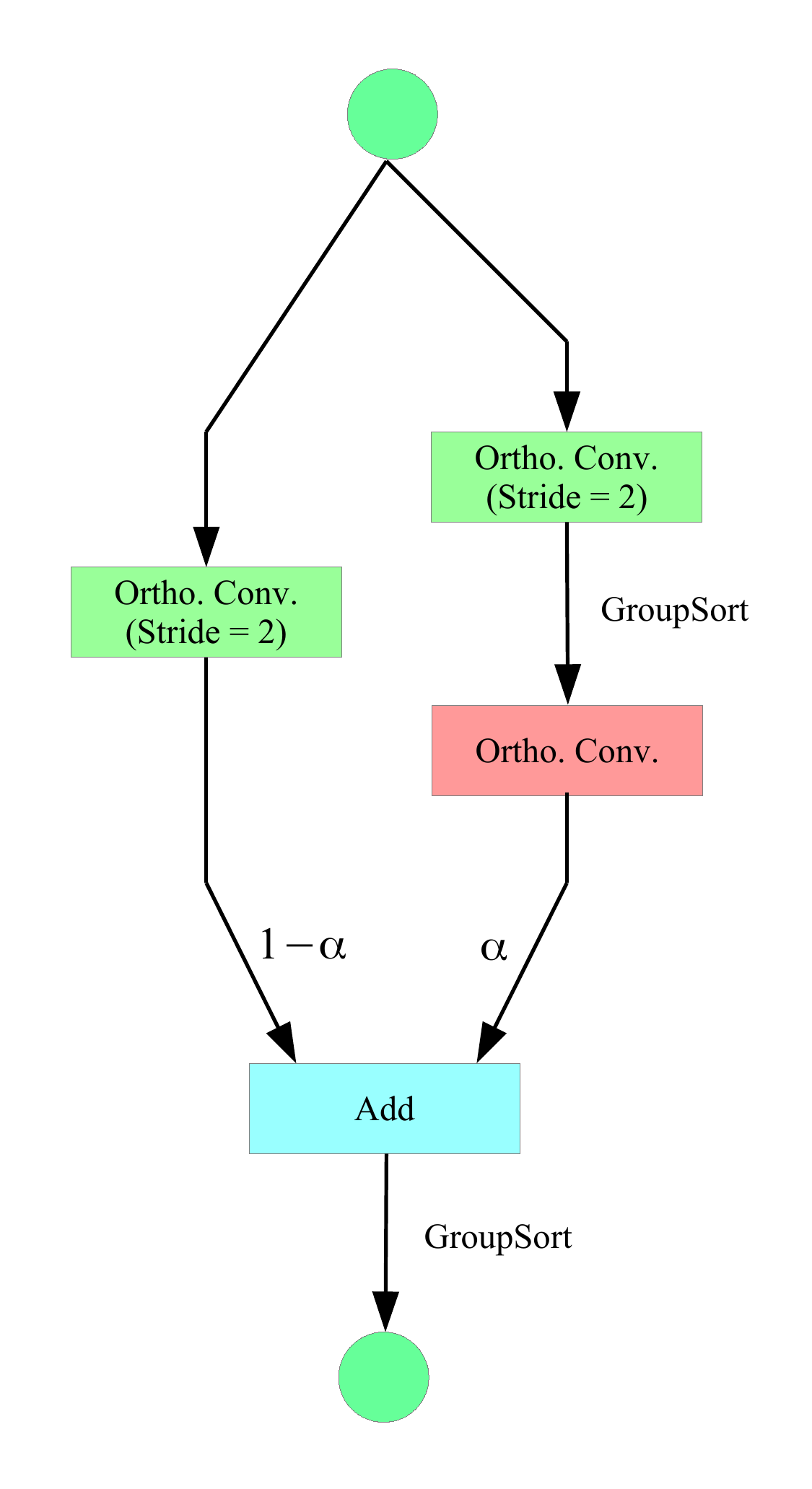}
        \caption{Strided additive block.}
        \label{subfig:resnet-strided}
    \end{subfigure}
    \hfill
    \begin{subfigure}[b]{0.22\textwidth}
    \centering
        \includegraphics[trim={0.6cm 0.8cm 0.6cm 0.8cm },clip,width=\textwidth]{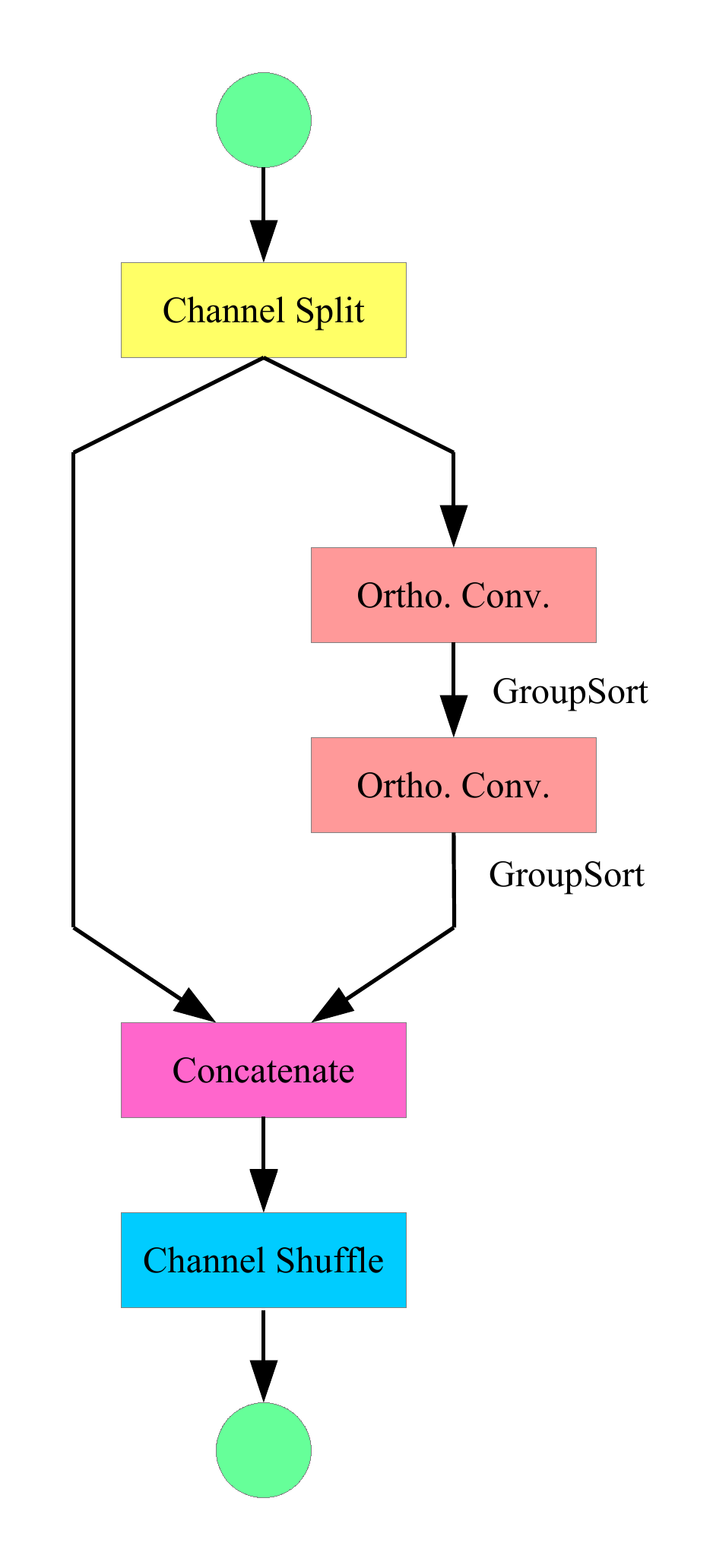}
        \caption{Basic shuffling block.}
        \label{subfig:shufflenet}
    \end{subfigure}
    \hfill
    \begin{subfigure}[b]{0.27\textwidth}
    \centering
        \includegraphics[trim={0.6cm 0.8cm 0.6cm 0.8cm },clip,width=\textwidth]{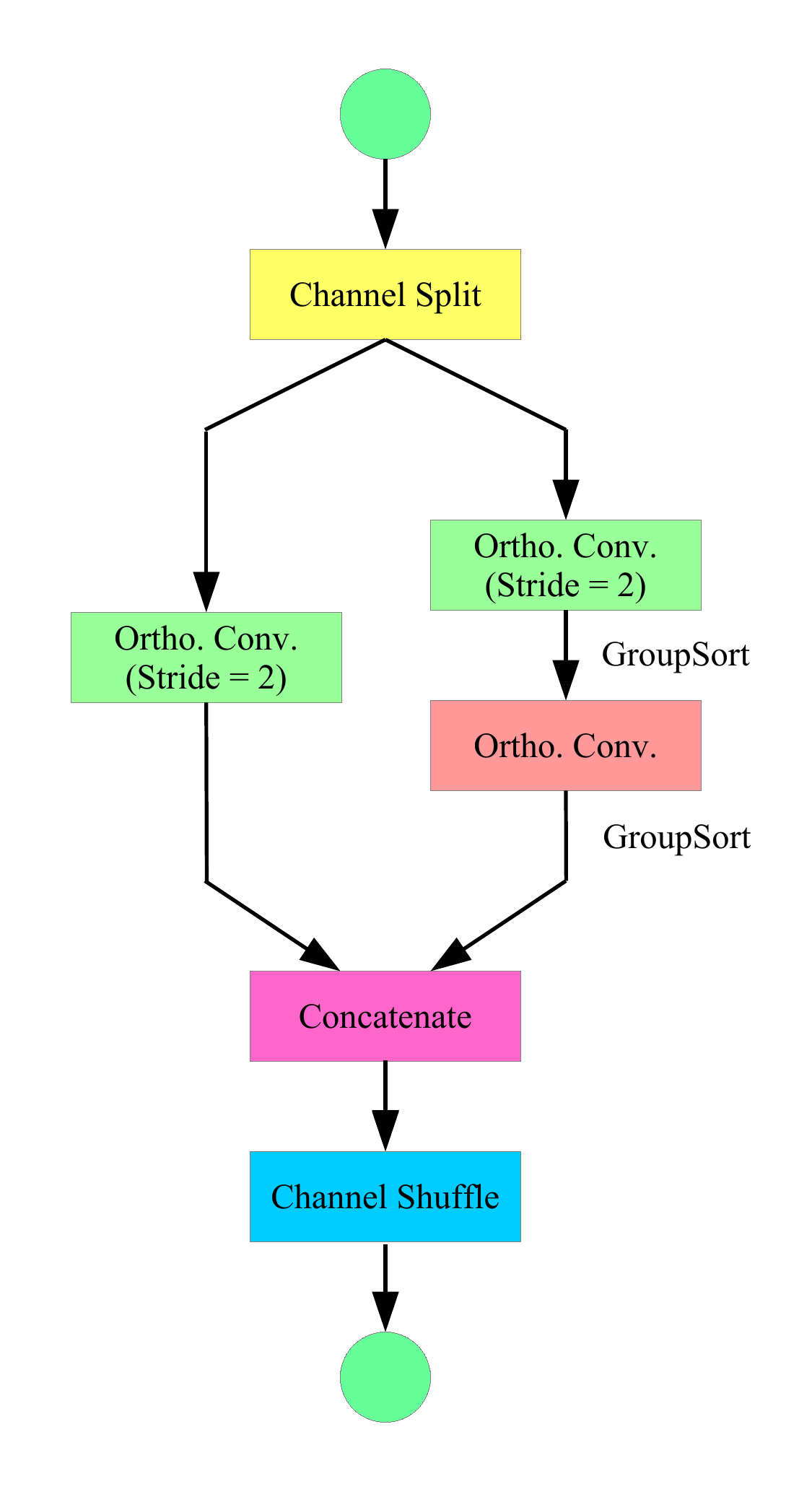}
        \caption{Strided shuffling block.}
        \label{subfig:shufflenet-strided}
    \end{subfigure}
\caption{{\bf Variants of residual blocks.}
In our experiments, we combine {\bf (a) \& (b)} to construct an {\em orthogonal \ResNet},
and {\bf (c) \& (d)} to construct an {\em orthogonal \ShuffleNet}.
In \Cref{prop:lipschitz-resnet}, we prove the Lipschitzness of these building blocks. 
Since composition of Lipschitz functions is still Lipschitz, it implies that a network constructed by these building blocks is also Lipschitz. 
}
\label{fig:residual-blocks}
\end{figure}

\begin{proof}[Proof for \Cref{prop:paraunitary-initialization}]
In order to prove \Cref{eq:paraunitary-initialization}, it suffice to show that
\begin{equation}
\FilterX{V}{z; \matrixSup{U}{-\ell}} \mymatrix{Q} \FilterX{V}{z^{-1}; \matrixSup{U}{\ell}} = \mymatrix{Q},
\end{equation}
and \Cref{eq:paraunitary-initialization} will reduce recursively to the orthogonal matrix $\mymatrix{Q}$. For simplicity, we rewrite $\matrixSup{U}{-\ell}$ as $\mymatrix{L}$ and $\matrixSup{U}{\ell}$ as $\mymatrix{R}$, by which we have $\mymatrix{L} = \mymatrix{Q} \mymatrix{R}$ (or $\mymatrix{R} = \mymatrix{Q}^{\top} \mymatrix{L}$) and we aim to prove $\FilterX{V}{z; \mymatrix{L}} \mymatrix{Q} \FilterX{V}{z^{-1}; \mymatrix{R}} = \mymatrix{Q}$. By the definition of $\FilterX{V}{z; \cdot}$ in \Cref{eq:v-block}, we expand it as
\begingroup
\small
\begin{equation}
\begin{aligned}
& \qquad \FilterX{V}{z; \mymatrix{L}} \mymatrix{Q} \FilterX{V}{z^{-1}; \mymatrix{R}} = \left[ \left( \identity - \mymatrix{L} \mymatrix{L}^{\top} \right) + \mymatrix{L} \mymatrix{L}^{\top} z \right] \mymatrix{Q} \left[\left(\identity - \mymatrix{R} \mymatrix{R}^{\top} \right) + \mymatrix{R} \mymatrix{R}^{\top} z^{-1} \right] = \\
& \underbrace{\mymatrix{L} \mymatrix{L}^{\top} \mymatrix{Q} (\identity - \mymatrix{R} \mymatrix{R}^{\top} )}_{\textstyle \filterX{c}{-1} \mathstrut} z + \underbrace{(\identity - \mymatrix{L} \mymatrix{L}^{\top} ) \mymatrix{Q} (\identity - \mymatrix{R} \mymatrix{R}^{\top} ) + \mymatrix{L} \mymatrix{L}^{\top} \mymatrix{Q} \mymatrix{R} \mymatrix{R}^{\top}}_{\textstyle \filterX{c}{0} \mathstrut} + \underbrace{(\identity - \mymatrix{L} \mymatrix{L}^{\top}) \mymatrix{Q} \mymatrix{R} \mymatrix{R}^{\top}}_{\textstyle \filterX{c}{1} \mathstrut} z^{-1}
\end{aligned}
\end{equation}
\endgroup
Therefore, we will need to show that $\filterX{c}{-1} = 0$, $\filterX{c}{1} = 0$ and $\filterX{c}{0} = \mymatrix{Q}$.

We first show that both $\filterX{c}{-1}$ for $z$ and $\filterX{c}{1}$ for $z^{-1}$ are zero matrices.
\begin{align}
\filterX{c}{-1} & =
\mymatrix{L} \mymatrix{L}^{\top}  \mymatrix{Q} \left(\identity - \mymatrix{R} \mymatrix{R}^{\top} \right) \\
& = \mymatrix{L} \mymatrix{L}^{\top} \mymatrix{Q} - \mymatrix{L} \mymatrix{L}^{\top} \mymatrix{Q} \mymatrix{R} \mymatrix{R}^{\top} \\
& = \mymatrix{L} ( \mymatrix{Q}^{\top} \mymatrix{L})^{\top}  - \mymatrix{L} ( \mymatrix{Q}^{\top} \mymatrix{L})^{\top} \mymatrix{R}  \mymatrix{R}^{\top} \\
& = \mymatrix{L} \mymatrix{R}^{\top} - \mymatrix{L} (\mymatrix{R}^{\top} \mymatrix{R}) \mymatrix{R}^{\top} \\
& = \mymatrix{L} \mymatrix{R}^{\top} - \mymatrix{L} \mymatrix{R}^{\top} = \zero, \\
\filterX{c}{1} & = 
(\identity - \mymatrix{L} \mymatrix{L}^{\top}) \mymatrix{Q} \mymatrix{R} \mymatrix{R}^{\top} \\
& = \mymatrix{Q} \mymatrix{R} \mymatrix{R}^{\top} - \mymatrix{L} \mymatrix{L}^{\top} \mymatrix{Q} \mymatrix{R} \mymatrix{R}^{\top} \\
& = (\mymatrix{Q} \mymatrix{R}) \mymatrix{R}^{\top} - \mymatrix{L} \mymatrix{L}^{\top} (\mymatrix{Q} \mymatrix{R}) \mymatrix{R}^{\top} \\
& = \mymatrix{L} \mymatrix{R}^{\top} - \mymatrix{L} ( \mymatrix{L}^{\top} \mymatrix{L} ) \mymatrix{R}^{\top} \\
& = \mymatrix{L} \mymatrix{R}^{\top} - \mymatrix{L} \mymatrix{R}^{\top} = \zero.
\end{align}
Lastly, we show that the constant coefficient $\filterX{c}{0}$ is equal to $\mymatrix{Q}$.
\begin{align}
\filterX{c}{0}
& = (\identity - \mymatrix{L} \mymatrix{L}^{\top} ) \mymatrix{Q} (\identity - \mymatrix{R} \mymatrix{R}^{\top} ) + \mymatrix{L} \mymatrix{L}^{\top} \mymatrix{Q} \mymatrix{R} \mymatrix{R}^{\top} \\
& = \mymatrix{Q} - \mymatrix{L} \mymatrix{L}^{\top} \mymatrix{Q} - \mymatrix{Q} \mymatrix{R} \mymatrix{R}^{\top} + 2 \mymatrix{L} \mymatrix{L}^{\top} \mymatrix{Q} \mymatrix{R} \mymatrix{R}^{\top} \\
& = \mymatrix{Q} - \mymatrix{L} \mymatrix{R}^{\top} - \mymatrix{L} \mymatrix{R}^{\top} + 2 \mymatrix{L} \mymatrix{R}^{\top} = \mymatrix{Q}
\end{align}
which completes the proof.
\end{proof}

\section{Supplementary Materials for Experiments}
\label{app-sec:experiments}

\subsection{Experimental Setup}
\label{app-sub2:training-details}

\textbf{Network architectures.}
For fair comparisons, we follow the architectures by \citet{trockman2021orthogonalizing} for KW-Large, ResNet9, WideResNet10-10 (i.e., shallow networks). For networks deeper than $10$ layers, we implement their architectures modifying from the Pytorch official implementation of ResNet. We find that it is crucial to replace the global pooling before fully-connected layers with a local average pooling with a window size of $4$. For the average pooling, we multiple the output with the window size to maintain its $1$-Lipschitzness. Other architectures, including ShuffleNet and plain convolutional network (ConvNet), are further modified from the ResNet, where only the skip-connections are changed or removed. We use the widen factor to indicate the number of channels for an architecture: the number of channels at each layer is defined as base channels multiply with the widen factor. The base channels are $16, 32, 64$ for three groups of residual blocks. More details of the ResNet architecture can be found in \url{https://github.com/pytorch/vision/blob/master/torchvision/models/resnet.py}.

\vspace{0.3em}
\textbf{Learning strategies.}
We use the CIFAR-10 dataset for all our experiments. All input images are normalized to $[0, 1]$ followed by standard augmentation, including random cropping and horizontal flipping. We use the Adam optimizer with a maximum learning rate of $10^{-2}$ coupled with a piece-wise triangular learning rate scheduler. We initialize all our \SCFac layers as permutation matrices. 

\subsection{Additional Empirical Results on Hyper-parameters Selection}
\label{app-sub2:initialization}

\textbf{Multi-class hinge loss.} 
Following previous works on Lipschitz networks~\citep{anil2019sorting,li2019preventing,trockman2021orthogonalizing}, we adopt the multi-class hinge loss in training. For each model, we perform a grid search on different margins $\epsilon_0 \in \{1\times10^{-3}, 2\times10^{-3}, 5\times10^{-3}, 1\times10^{-2}, 2\times10^{-2}, 5\times10^{-2}, 0.1, 0.2, 0.5\}$ and report the best performance in terms of robust accuracy. 
Notice that the margin $\epsilon_0$ controls the trade-off between clean and robust accuracy, as shown in \Cref{fig:comp_margin}.

\begin{figure}
    \centering
    \begin{tikzpicture}
        \begin{axis}[
            legend pos=outer north east,
            symbolic x coords={500, 200, 100, 50, 20, 10, 5, 2, 1},
            xtick=data,
            xlabel={Margin $\epsilon_0$ ($\times10^{-3})$},
            ylabel={Accuracy (\%)},
            width=7cm, height=5.5cm,
            legend style={font=\footnotesize},
            tick label style={font=\scriptsize},
            label style={font=\footnotesize},
            ]
            \addplot[
                red,
                mark=x] 
                coordinates{
                (500, 81.92) (200, 87.58) (100, 94.16) 
                (50, 95.19) (20, 96.79) (10, 96.95) 
                (5, 97.13) (2, 96.33) (1, 95.09)};
            \addplot[
                teal,
                mark=x]
                coordinates{
                (500, 78.19) (200, 82.30) (100, 86.34) 
                (50, 87.42) (20, 88.94) (10, 88.74) 
                (5, 88.54) (2, 87.96) (1, 86.40)};
            \addplot[
                blue,
                mark=*] 
                coordinates{
                (500, 69.07) (200, 70.69) (100, 75.41) 
                (50, 75.69) (20, 75.17) (10, 73.17)
                (5, 70.55) (2, 64.85) (1, 56.83)};
            \legend{Train, Test, PGD}
        \end{axis}
    \end{tikzpicture}
    \caption{{\bf Effect of the Lipschitz margin $\epsilon_0$ for WideResNet22-10}. It shows a trade-off between clean and robust accuracy with different margins for multi-class hinge loss. As shown, the training and test accuracy become higher with larger margin, but the robust accuracy decreases after $\epsilon_0=0.1$.}
    \label{fig:comp_margin}
\end{figure}
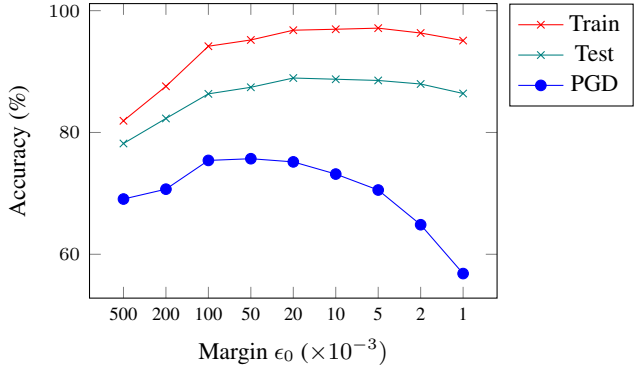

\vspace{0.3em}
\textbf{Network depth and width.}
Exact orthogonality is criticized for harming the expressive power of neural networks, and we find that increasing network depth/width can partially compensate for such loss.
In \Cref{tab:depth-1}, we perform a study on the impact of network depth/width on the predictive performance.
As shown, deeper/wider architectures consistently improve both the clean and robust accuracy for our implementation. However, the best robust accuracy is achieved by a $22$-layer network since we can afford a wide architecture for $34$-layer architecture.

\begin{table}
\centering
\caption{{\bf Comparison of different depth and width} on WideResNet (kernel size $5$). Some numbers are missing due to the large memory requirement (on Tesla V100 32G). The notation width factor indicates (channels = base channels $\times$ factor)}
\begin{tabular}{c|ccccc|ccccc}
    \toprule
    &  \multicolumn{10}{c}{10 layers} \\
    \midrule
    Width &  1 & 3 & 6 & 8 & 10 &  1 & 3 & 6 & 8 & 10 \\
    \midrule
    & \multicolumn{5}{c}{Clean} & \multicolumn{5}{c}{PGD (\multirow{1}{*}{$\frac{36}{255}$})}  \\
    \midrule
    Ours  & 79.96 & 84.17 & {84.96} & 84.61 & 84.09 & 65.92 & 69.70 & 72.18 & 72.51 & {74.29} \\
    Cayley  & 77.88 & 82.14 & 82.56 & \textbf{85.53} & {85.01} & 66.65 & 73.06 & 74.33 & 75.66 & {76.13} \\
    RKO & 81.37 & 83.55 & 84.67 & 85.18 & 84.62 & 70.55 & 74.44 & 76.41 & 76.65 & \textbf{77.02} \\
    \midrule
    \midrule
    &  \multicolumn{10}{c}{22 layers} \\
    \midrule
    Width &  1 & 3 & 6 & 8 & 10 & 1 & 3 & 6 & 8 & 10 \\
    \midrule
      & \multicolumn{5}{c}{Clean} & \multicolumn{5}{c}{PGD (\multirow{1}{*}{$\frac{36}{255}$})}  \\
    \midrule
    Ours & 79.90 & 82.22 & 87.21 & \textbf{88.10} & 87.82 & 67.95 & 70.88 & 74.30 & 75.12 & \textbf{76.46} \\
    Cayley & 79.11 & 84.82 & 85.85 & - & - & 69.79 & 65.61 & 74.81 & - & -\\
    RKO & 82.71 & 84.19 & 84.33 & 84.55 & - & 72.40 & 74.36 & 75.66 & 76.41 & - \\
    \midrule
    \midrule
    &  \multicolumn{10}{c}{34 layers} \\
    \midrule
    Width &  1 & 3 & 6 & 8 & 10 &  1 & 3 & 6 & 8 & 10 \\
    \midrule
      & \multicolumn{5}{c}{Clean} & \multicolumn{5}{c}{PGD (\multirow{1}{*}{$\frac{36}{255}$})}  \\
    \midrule
    Ours & 81.24 & 88.17 & \textbf{88.92} & - & - & 69.21 & 71.85 & \textbf{75.09} & - & - \\
    Cayley & 82.46 & 84.29 & - & - & - & 71.27 & 74.73 & - & - & - \\
    RKO & 81.51 & 83.24 & 83.92 & - & - & 71.38 & 73.84 & 75.03 & - & - \\
    \bottomrule
\end{tabular}
\label{tab:depth-1}
\end{table}

\vspace{0.3em}
\textbf{Initialization methods.}
In \Cref{prop:paraunitary-initialization}, we have shown how to initialize our orthogonal convolutional layers as orthogonal matrices.
In \Cref{tab:result-init}, we perform a study on different initialization methods, including identical, permutation, uniform, and torus~\citep{henaff2016recurrent,helfrich2018orthogonal}. We find that permutation initialization works the best for WideResNet22-10, while all methods perform similarly in shallower WideResNet10-10. Therefore, we use permutation initialization for all other experiments.

\begin{table}
\centering
\caption{{\bf Comparisons of various initialization methods} on WideResNet (kernel size $5$).}
\begin{tabular}{l|cc|cc}
    \toprule
    \multirow{1}{*}{Initialization} & \multicolumn{2}{c|}{WideResNet10-10} & \multicolumn{2}{c}{WideResNet22-10} \\
    \cmidrule{2-5}
     & Clean & PGD & Clean & PGD \\
    \midrule
    uniform & \textbf{83.58} & 73.20 & 87.55 & 75.71 \\
    torus & 82.40 & 72.50 & \textbf{88.12} & 75.43 \\
    permutation & 83.18 & 73.16 & {87.82} & \textbf{76.46} \\
    identical & 83.29 & \textbf{73.49} & {87.82} & {75.49} \\
    \bottomrule
\end{tabular}
\label{tab:result-init}
\end{table}

\section{Boarder Impact}
\label{app-sec:impact-statement}

Our work lies in the foundational research on neural information processing. Specifically, we establish the equivalence between orthogonal convolutions in neural networks and paraunitary systems in signal processing. Our presented orthogonal convolutional layers are plug-and-play modules that can replace various convolutional layers in neural networks. Consequently, our modules are applicable in Lipschitz networks for adversarial robustness, recurrent networks for learning long-term dependency, or flow-based networks for effortless reversibility.

The vulnerability of neural networks raises concerns about their deployment in security-sensitive scenarios, such as healthcare systems or self-driving cars. In our experiment, we demonstrate a successful application of orthogonal convolutions in learning robust networks. These networks achieve high robust accuracy without additional techniques such as adversarial training or randomized smoothing. Therefore, our research contributes to the robustness learning of neural networks and potentially leads to their broader deployment. 

As an expense, our layers are memory and computationally more expensive than traditional layers. The overhead to the already expensive cost exacerbates the concerns on the efficacy of neural networks. Therefore, balancing between robustness and efficiency is an important research topic that requires more research in the future. In this work, We develop more efficient implementation than previous approaches, narrowing the gap between these two conflicting goals.

\end{document}